\documentclass[10pt,reqno,final]{amsart}

\usepackage[foot]{amsaddr}

\usepackage[english]{babel}
\usepackage[utf8]{inputenc}
\usepackage[T1]{fontenc}

\usepackage[margin=1in, right=1in, left=1in]{geometry}

\usepackage{graphicx} 

\usepackage{hyperref}
\hypersetup{
    colorlinks,
    citecolor=blue,
    filecolor=blue,
    linkcolor=blue,
    urlcolor=blue
}

\usepackage{amsthm,amssymb,mathtools,fixmath}
\usepackage{esint}
\usepackage{enumerate}
\usepackage{float}
\usepackage{comment}
\usepackage{listings}
\usepackage[dvipsnames]{xcolor}
\usepackage{algorithm} 
\usepackage{algpseudocode}

\usepackage{lmodern}
\usepackage{enumitem}
\setenumerate{nolistsep}

\usepackage{siunitx}[=v2]
\usepackage[textsize=small,textwidth=.8in,disable]{todonotes}
\newtheorem{theorem}{Theorem}[section]

\newtheorem{proposition}[theorem]{Proposition}

\newtheorem{corollary}[theorem]{Corollary}

\newtheorem{remark}[theorem]{Remark}

\numberwithin{equation}{section}

\newcommand{\de}{\mathop{}\!\mathrm{d}}

\newcommand{\R}{\mathbb{R}}
\renewcommand{\S}{\mathbb{S}}

\DeclarePairedDelimiter{\norm}{\lVert}{\rVert}
\DeclarePairedDelimiter{\abs}{\lvert}{\rvert}
\DeclarePairedDelimiter{\pair}{\langle}{\rangle}

\DeclareMathOperator{\tr}{tr}

\DeclareMathOperator{\supp}{supp}
\DeclareMathOperator{\rank}{rank}
\DeclareMathOperator{\range}{range}
\DeclareMathOperator{\diag}{diag}

\newcommand{\bld}[1]{\boldsymbol{#1}}

\newcommand{\NN}{\mathcal{S}}
\newcommand{\M}{\mathcal{M}}
\newcommand{\V}{\mathcal{V}}
\renewcommand{\H}{\mathcal{H}}
\newcommand{\Four}{\mathcal{F}}
\newcommand{\Rad}{\mathcal{R}}

\definecolor{darkgreen}{rgb}{0.0, 0.5, 0.0} 
\definecolor{orange}{rgb}{0.8, 0.33, 0.0}

\title{Nonuniform random feature models using derivative information}\thanks{This manuscript has been authored by UT-Battelle, LLC, under contract DE-AC05-00OR22725 with the US Department of Energy (DOE). The US government retains and the publisher, by accepting the article for publication, acknowledges that the US government retains a nonexclusive, paid-up, irrevocable, worldwide license to publish or reproduce the published form of this manuscript, or allow others to do so, for US government purposes. DOE will provide public access to these results of federally sponsored research in accordance with the DOE Public Access Plan.}

\author{Konstantin Pieper \and Zezhong Zhang \and Guannan Zhang}
\date{\today}

\begin{document}

\maketitle

\begin{abstract}
We propose nonuniform data-driven parameter distributions for neural network initialization based on derivative data of the function to be approximated.
These parameter distributions are developed in the context of non-parametric regression models based on shallow neural networks, and compare favorably to well-established uniform random feature models based on conventional weight initialization.
We address the cases of Heaviside and ReLU activation functions, and their smooth approximations (sigmoid and softplus), and use recent results on the harmonic analysis and sparse representation of neural networks resulting from fully trained optimal networks.
Extending analytic results that give exact representation, we obtain densities that concentrate in regions of the parameter space corresponding to neurons that are well suited to model the local derivatives of the unknown function.
Based on these results, we suggest simplifications of these exact densities based on approximate derivative data in the input points that allow for very efficient sampling and lead to performance of random feature models close to optimal networks in several scenarios.
\end{abstract}

\section{Introduction}
Multivariate non-parametric regression is a widely used tool in machine learning. The two main approaches used in this context are reproducing kernels, which can be analyzed though the theory of reproducing kernel Hilbert spaces~\cite{Wendland_2004,scholkopf2018learning} (RKHS), and deep neural networks, where a theoretical foundation is subject of ongoing research; see, e.g.,~\cite{ShallowDeep:2017,E_overview:2020,EWojtowytsch:2020,SchmidtHieber:2020,Bartlett_Montanari_Rakhlin_2021,ma2022barron,PaNo:2022}.
However, restricted classes of neural networks, such as shallow or one-hidden-layer networks, can be analyzed thoroughly through the lens of infinite feature spaces~\cite{bengio2006convex,RossetSwirszczSrebroZhu:2007,Neyshabur:2015}, harmonic analysis of ridge functions~\cite{murata1996integral,candes1999harmonic,sonoda2017neural}, and ridge splines~\cite{kainen2010integral,klusowski2018approximation,Petrosyan:2020}.
In particular, through the choices made during the training of the network (modifications of the gradient-based training, early stopping, random initialization, and explicit regularization such as weight decay), the properties of the network can be understood either through the lens of sparse regularization~\cite{bach2017breaking,ongie2019function,parhi2021neural,chizat2022sparse,Pieper2022}, or reproducing kernels associated with the activation function and the distributions used to initialize the weights~\cite{leroux07a,neal2012bayesian,RahimiRecht:08,bach2023relationship} or both~\cite{neumayer2024effect}.
However, for the simpler class of random feature models, the precise choice of the distribution used to sample the inner weights is crucial, since it determines the performance of the method to a much larger degree as if all weights are trained after random initialization.

In this work, we investigate simple data-driven strategies that help to sample weights in a way that they will be nonuniformly distributed to favor the areas in weight space that are most likely to be necessary for good approximation with a small network. To demonstrate this, we consider a regression task \(f(x_k) \approx y_k\), for a shallow neural network model \(f\colon \R^d \supset X \to \R\) and assume that we do not only have access to input data \(x_k\) and \(y_k\), \(k=1,\ldots, K\) but also to derivative information of the model to be learned, e.g., \(g_k \approx \nabla f(x_k)\).

It is realistic to assume that such data is available in scientific regression applications, and gradient data plays a major role in tasks such as denoising score matching~\cite{DenoisingScore:2011}, optimal control with optimal value function approximation~\cite{HJBParticleNetwork:2021,KunischWalter:2023}, Bayesian optimization~\cite{BayesianOpt:2017} or nonlinear dimensionality reduction~\cite{DIPNet:2022,NonlinDR:2022}.
In contrast to derivative learning~\cite{GALLANT:1992} or unsupervised score learning~\cite{AapoScore:2005}, where the derivative of the network is learned without direct data, we assume that such data is available (or can be cheaply computed or estimated by another method).
However, the data set is not (necessarily) enhanced by the derivative data to improve fidelity~\cite{BayesianOpt:2017,SobolevTraining:2017}, but mainly to design data driven nonuniform sampling densities for hidden weight initialization.
Thus, in the context of weight initialization strategies for neural networks (e.g., \cite{WeightInitReview:2022}), we are proposing a hybrid data-driven and randomized strategy. In contrast to existing randomized approaches using derivative information (e.g., \cite{WeightInitHessian:2021}), we are incorporating the derivative data to favor weights that likely yield a good model.
We demonstrate that for random feature models without nonlinear training of inner weights, this yields an improvement in terms of the approximation quality for given number of weights.

\subsection{Framework}
In this work, we build on the established frameworks that bridge the transition between fully and partially trained shallow neural networks.
For ease of exposition, we only consider the empirical least squares loss
\begin{equation}
  \label{eq:empirical_loss}
  L(f; \bld{x}, \bld{y})
  = \frac{1}{2K}\sum_{k=1}^K (f(x_k) - y_k)^2,
\end{equation}
although the presented techniques can be easily generalized to different loss functions.
In particular, since we consider the case where data values for the gradient \(g_k = \nabla f(x_k)\) are available, these could be easily added to the objective; see~\ref{sec:discussion}.
For modeling the function, we consider function approximation using shallow networks of the form
\begin{equation}
\label{eq:shallow_intro}
x \mapsto
\NN(x;\bld{A},\bld{b},\bld{c})
= \bld{c}^T \sigma(\bld{A} x + \bld{b})
= \sum_{n=1}^N c_n \sigma(a_n\cdot x + b_n),
\end{equation}
where \(\bld{A} = (a_n)_{n=1,\ldots,N}\), \(\bld{b} = (b_n)_{n=1,\ldots,N}\), are the inner weights and offsets (also called ``biases'') and \(\bld{c} = (c_n)_{n=1,\ldots,N}\) are the linear or outer weights, and \(\sigma\), to be specified below, is the activation function.

To provide a systematic framework for network training, we first separate the network parameters into outer weights \(c_n\), which enter the model linearly, and inner weights \(\omega_n = (a_n, b_n)\), which enter nonlinearly.
To restrict the complexity of the neural network (i.e.\ be able to control the smoothness of the modeled function) only in terms of the outer weights, we restrict the inner weights to a bounded set \(\omega_n \in \Omega \subset \R^{d+1}\).
In particular, we consider
\begin{equation}
\label{eq:parameter_set}
  \Omega = \S^{d-1} \times \R = \{\,(a,b) \;|\; \norm{a} = 1 \,\}.
\end{equation}
This is motivated by positively homogeneous activation functions such as Heaviside, ReLU, or ReLU\(^{s-1}\)
for \(s > 1\)~\cite{parhi2021neural,ReLUkApproximation:2024},
where such a limitation does not restrict the expressive power of the network;
see section~\ref{sec:problem_notation}.
We also consider activation functions that can be interpreted as smooth approximations of these nonsmooth activation functions; see section~\ref{sec:discussion}.
To train an optimal network we then solve the variable term nonlinear regression problem that combines the empirical loss with a sparse regularizer
\begin{equation}
\label{eq:psi_regression}
\min_{N\geq 0, \bld{c} \in \R^N, \bld{\omega}\in \Omega^N}
L(\NN(\cdot;\bld{\omega},\bld{c}); \bld{x}, \bld{y})
+ \alpha \sum_{n=1}^N \psi(\abs{c_n}).
\end{equation}
Here, the number of neurons is adaptively chosen during the optimization, and the width of the network \(N\) is only indirectly controlled through the regularization term with the regularization parameter \(\alpha > 0\), which is often the \(\ell_1\) norm with \(\psi(t) = t\).
In this case, for homogeneneous activation function, this training objective can be equivalently rewritten as a specific form of ``weight-decay'' regularization~\cite{savarese2019, ongie2019function,parhi2021neural}.
To improve the sparsity (i.e.\ reduce \(N\)) at the same approximation quality, we also consider concave subadditive penalties with \(\psi'(t) = 1\) as in~\cite{Pieper2022}.
These variable term finite optimization problems are equivalent formulations of nonparametric sparse regression problems over infinite feature spaces~\cite{RossetSwirszczSrebroZhu:2007,bengio2006convex}, which allow for rigorous analysis of expressivity, generalization, and sidestep issues of local minimizers that affect finite neural network training methods; see section~\ref{sec:infinite_feature_regression}.
In particular, for \(\ell_1\)-regularization, finite global solutions can be approximated reliably, and for other regularizers certain local solutions with global approximation properties can be obtained.
However, training these models requires not only local gradient descent (for a nonconvex, nonlinear optimization problem), but also ways of determining the appropriate finite parameter \(N\) using node insertion and deletion strategies using boosting, coordinate descent, or generalized conditional gradient methods~\cite{mazumder2011sparsenet,bredies2013inverse,boyd2017alternating,flinth2019linear,pieper_walter_2021}.
Alternatively, for homogeneous activation functions, specific forms of conventional gradient based training with random weight initialization and weight decay can be considered as over-parameterized particle based methods to compute the global optimum of~\eqref{eq:psi_regression}; cf., e.g.,~\cite{chizat2018global,chizat2022sparse,Rotskoff:2022}.

It is well-known that many of the advantages of neural networks/infinite feature space representations can be realized by the simpler technique of random feature models.
Here, the inner weights are sampled in an i.i.d.\ fashion from some \emph{a-priori} chosen probability distribution \(M\) on the parameter set \(\Omega\), and not trained by gradient descent.
Often, these are simple distributions, such as uniform or Gaussian distributions~\cite{RahimiRecht:08,leroux07a,RudiRosasco:2017,bach2023relationship,zhang2023transnet}, where the mean and variance is solely informed by the network size, employed activation function, and domain of interest for the input parameters.
Here, we consider the uniform distribution on the bounded parameter set
\begin{equation}
\label{eq:parameter_set_R}
  \Omega_R = \S^{d-1} \times \R = \{\,(a,b) \;|\; \norm{a} = 1, \abs{b} \leq R\,\}.
\end{equation}
A uniform density in \(\Omega_R\) of the parameter $(a,b)$ results in associated inflection hyperplanes $\{\,x\colon a\cdot x + b = 0\,\}$ that also have a constant probability density in the unit ball~\cite{zhang2023transnet}. The distribution of the location of these hyperplanes is crucial to the expressive power of a random feature model, since all of the nonlinear variation of the feature functions \(\varphi(x) = \sigma(a \cdot x + b)\)
is near the inflection hyperplanes. A uniform distribution on inflection hyperplanes will ensure all parts of the domain of interest, contained in a ball of radius \(R\), are covered by the high-variation portion of the feature functions with the same probability. 
We note that this uniform strategy results from an initial i.i.d.\ normal distribution on the weight \(a \sim \mathcal{N}(0,I_d)\) together with subsequent normalization to \(\norm{a}=1\) together with \(b \sim \mathcal{U}([-R,R])\).

After sampling the inner weights, only the outer weights are trained; since they enter linearly, we can interpret this as a random feature model
\begin{equation*}
\label{eq:random_feature}
\NN(x;\bld{\omega},\bld{c})
= \bld{c}^T \bld{\varphi}(x;\bld{\omega})
= \sum_{n=1}^N c_n\varphi_n(x,\omega_n),
\quad\text{where } \varphi_n(x,\omega_n)
= \sigma(a_n \cdot x + b_n), 
\quad (a_n,b_n) \sim \de M(a,b).
\end{equation*}
This implies that more efficient solution techniques than gradient descent are available for training, such as matrix factorization to solve the normal equations associated to a corresponding linear regression problem.  Since \(\NN(\bld{c},\bld{\omega})\)
is linear in the outer weights \(\bld{c}\), and typical loss functions are convex in \(f\), we can always reliably find their optimal values, given the inner weights.
In fact, for the regularized least squares minimization
\begin{equation}
\label{eq:ridge}
\min_{\bld{c}\in\R^N}
L(\NN(\cdot;\bld{\omega},\bld{c});\bld{x},\bld{y})
+ \frac{\alpha N}{2}\sum_{n=1}^N \abs{c_n}^2
\end{equation}
and the associated feature matrix
\(\Phi_{k,n} = \varphi(x_k, \omega_n)\), the regularized least squares minimizer is given as
\[
\bar{\bld{c}}_\alpha 
= \left(\alpha N + \frac{1}{K}\Phi^\top \Phi \right)^{-1} \left(\frac{1}{K}\Phi^\top y\right)
= \frac{1}{N}\Phi^\top \left(\alpha K + \frac{1}{N}\Phi \Phi^\top \right)^{-1} y
\]
for a regularization parameter \(\alpha > 0\). In other words, we are performing kernel ridge regression with respect to the finite rank kernel 
\begin{align}
\label{eq:finite_rank_kernel}
k_{M,N}(x_k,x_{k'}) 
= \frac{1}{N} [\Phi \Phi^\top]_{k,k'}
= \frac{1}{N} \sum_{n=1}^N \varphi(x_k;\omega_n)\varphi(x_{k'};\omega_n)
\end{align}
that is determined by the distribution \(\de M\) and the finite truncation parameter \(N\).
However, in order to match the accuracy of fully trained sparse feature models, random feature models have to rely on a large degree of over-parametrization, i.e.\ \(N\) must be chosen much larger.
The degree to which this increase in parameters has to happen can be arbitrarily large, depending on the structure of the function \(f\). This stems from the adaptivity enabled by the optimal choice of the inner weights \(\omega_n\), which can not be exploited in uniform random feature models.
Our goal is to develop nonuniform probability measures \(\de M\) informed by derivative values of \(f\), which narrow this gap and enable random feature models to take advantage of some degree of adaptivity.
Before going into details, we illustrate the types of adaptivity that distinguish sparse from random feature models.

\subsection{Adaptivity}
\label{sec:intro_illustrate}
To illustrate the differences between sparse and random feature models, we start with a simple experiment with the classical sigmoid unit \(\sigma(t) = (1+\exp(-t/\delta))^{-1}\), scaled by a scaling parameter \(\delta < 1\), which approximates the classical Heaviside function for \(\delta \to 0\).
We learn a simple Gaussian profile from \(K=1000\) noisy data points on the interval \([-1,1]\). Here, we fix \(\delta = 1/80\) and sample the inner weights uniformly from \(\Omega_1\), so that the random inflection point \(x_n = -b_n/a_n\) is uniform in the data interval.
The regularization parameters for both experiments are selected with cross-validation based on a test set; see section~\ref{sec:example_1d} for quantitative results.
\begin{figure}
    \includegraphics[trim={0 0 0 8cm},clip,width=0.32\textwidth]{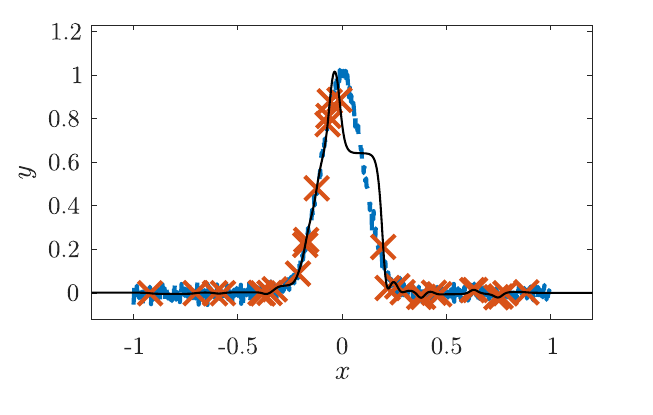}
    \includegraphics[trim={0 0 0 8cm},clip,width=0.32\textwidth]{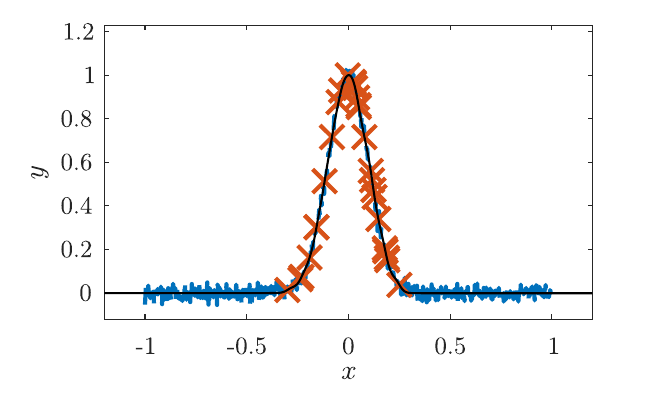}
    \includegraphics[trim={0 0 0 8cm},clip,width=0.32\textwidth]{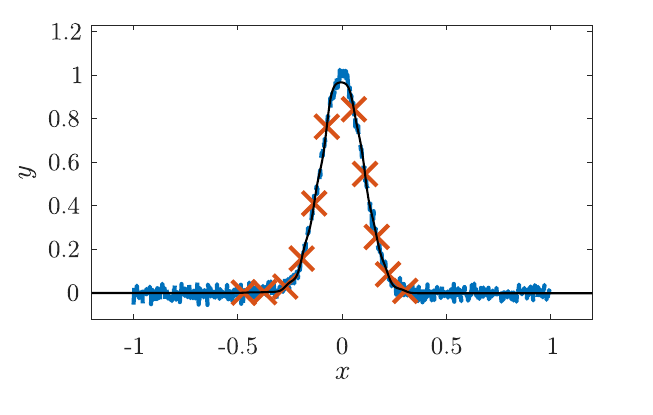}
    \caption{Comparison of uniform (left, \(N=30\)) and nonuniform (middle, \(N=30\)) random feature regression and sparse infinite feature regression (right, \(N=11\)) with for a scaled sigmoidal activation function \(\sigma(t) = 1/(1 + \exp(-t/\delta))\) with \(\delta = 1/80\). The noisy data values are dashed blue, the model function is the black line, and the inflection points \(x_n = -b_n/a_n\) of the sigmoids are orange crosses.
    The nonuniform random and sparse feature model are able to allocate more inflection points to regions of space with high variability of \(f\).}
    \label{fig:intro_spatial_adapt}
\end{figure}
From the representative results (one typical realization of the random model) in Figure~\ref{fig:intro_spatial_adapt}, we can see that the sparse feature model (fully trained network) can take advantage of the spatial localization of the slope of the function that is learned. On the other hand, the uniform random feature model is agnostic of the properties of the function, and selects the inflection points uniformly.
As a result, even with three times the number of units compared to the optimized model, the error is still noticeable.
Although the random feature model is trained by a simple normal equation linear solve based on a matrix factorization with orders of magnitude less effort than the fully trained network, this example illustrates the lack of spatial adaptivity with a uniform distribution.
In higher dimensions, this effect persists, and random feature model training by direct factorization can become more prohibitive, due to more than quadratic complexity of the linear solve in \(N\), and it is less feasible to accept the increased cost from redundant or unnecessary features. 

Aside from the lack of spatial adaptivity, a lack of directional adaptivity even more pronounced in higher dimensions, and can be the decisive factor in how much the curse of dimensionality can be mitigated. Usually, the inner weight is chosen a way such that each entry of \(a_n\) is independent of each other, e.g., i.i.d.\ Gaussian. In the case of the uniform distribution on~\eqref{eq:parameter_set_R} the kernel~\eqref{eq:finite_rank_kernel} has additional structure, and for nonsmooth homogeneous activation functions it has been shown that it converges for \(N\to \infty\) to a sum of polynomial terms and a radial kernel~\cite{leroux07a,bach2017breaking} related to polyharmonic splines.
We discuss this in detail in section~\ref{sec:uniform_random_features}, also for the case of smoother activation functions.
As a result, the random feature model can not adapt to anisotropic functions that vary more in some directions than others, as a result of having to sample the inner weight without knowledge of the function.
Consider the planar wave
\begin{equation}
\label{eq:example_1d2d}
f(x_1,x_2) = \sin(5 z_1),
\quad\text{where } z_1 = x_1 - \sqrt{2}x_2.
\end{equation}
After a linear coordinate transform, the function depends only on one variable, and is thus very easy to approximate. A shallow network with all the inner weights as a multiple of the direction \((1,-\sqrt{2})^T\) can take full advantage of that.
However, a random feature model has to treat this the same way as an isotropic two-dimensional function.
Figure~\ref{fig:intro_direct_adapt} illustrates that a fully trained network does learn this direction from the data, and is thus significantly more accurate than a random feature model even with reduced number of neurons; see section~\ref{sec:example_2d} for quantitative results.
\begin{figure}
\includegraphics[width=0.32\textwidth]{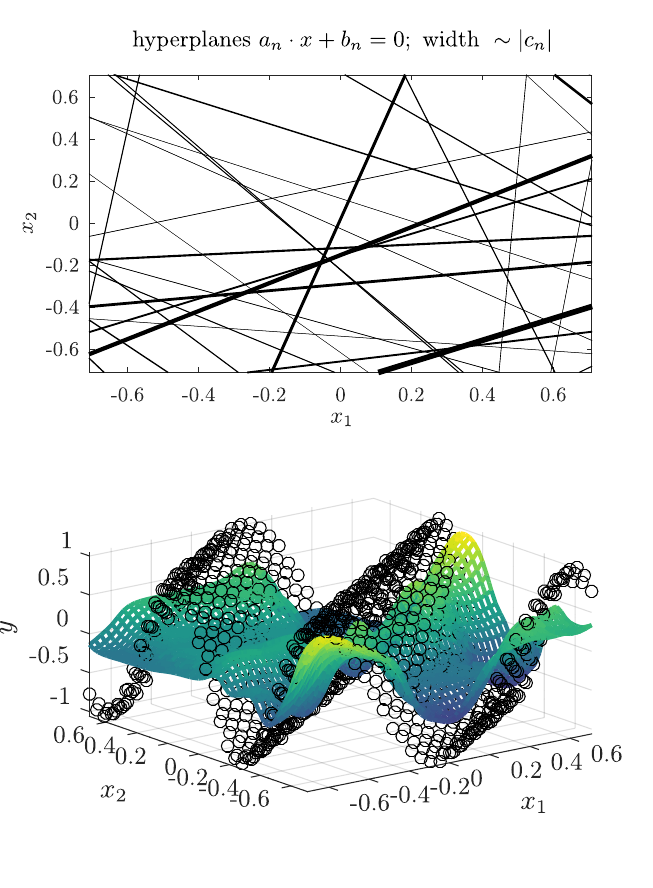}
\includegraphics[width=0.32\textwidth]{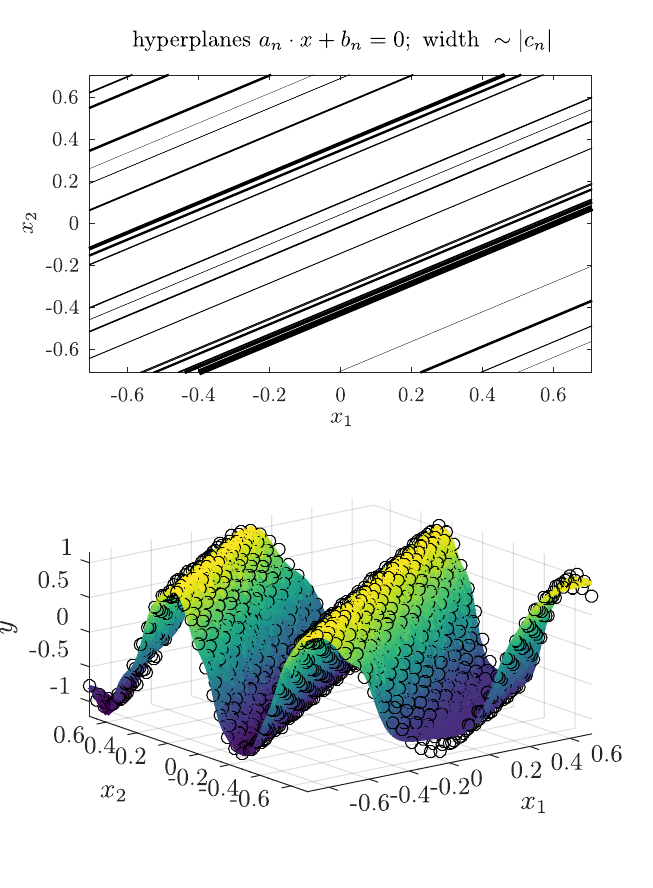}
\includegraphics[width=0.32\textwidth]{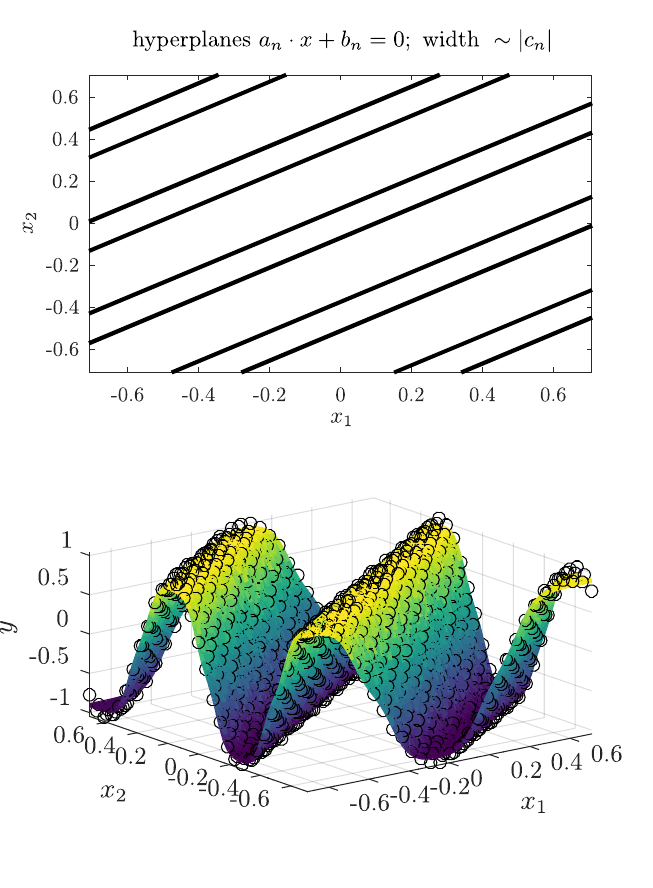}
\caption{
Comparison of uniform (left, \(N=30\)) and nonuniform (middle, \(N=30\)) random feature regression and sparse infinite feature regression (right, \(N=11\)) with for a scaled sigmoidal activation function \(\sigma(t) = 1/(1 + \exp(- t/\delta))\) with \(\delta = 1/40\). The data values are black circles and the model function is the surface (bottom plot); the inflection hyperplanes are visualized in the top figure.
The nonuniform random and sparse feature model are able to orient inflection hyperplanes with directions of variability of \(f\).
}
\label{fig:intro_direct_adapt}
\end{figure}

\subsection{Nonuniform random feature models}

The idea how to obtain nonuniform random feature models with spatial and directional adaptivity is motivated by the theoretical framework of infinite feature models.
For a large class of functions \(f\) and activation functions \(\sigma\) and different dictionaries \(\Omega\), exact representations have been obtained~\cite{murata1996integral,candes1999harmonic,sonoda2017neural,ongie2019function,Petrosyan:2020}. For smooth functions in our setting, an exact representation is a density \(c_f \in L^2(\Omega)\) (or in a space of measures; see section~\ref{sec:infinite_feature_regression}) and a low-degree polynomial \(p_0\) with
\[
f(x) = p_0(x) + \int_\Omega \varphi(x;\omega) c_f(\omega) \de \omega,
\quad\text{for all } x \in X
\]
and we summarize some existing results for non-smooth activation functions in Theorem~\ref{thm:exact_rep}.
Since these representations are unique over a space of (anti-)symmetric densities, these densities closely predict where the weights \(\omega_n\) of trained models will lie, or more precisely where the empirical measure \(\sum_{n=1}^N c_n \delta_{\omega_n}\) associated to a trained network will concentrate its mass.

The motivation for our developments is: given a function \(f\) that is approximately represented by the density \(c_f\) we would like to sample the weights of the network according to the probability density
\[
\de M(\omega) 
= \frac{1}{m_f} \abs{c_f(\omega)}\de\omega,
\quad m_f = \int_\Omega \abs{c_f(\omega)}\de\omega.
\]
Corresponding formulas for nonsmooth activation functions have been been derived and can be used as a basis for sampling (e.g., using rejection sampling; see~\cite{NonparametricWeight:2013}). We extend this by deriving approximate representation formulas for smooth activation functions with small smoothing parameter; see Corollary~\ref{cor:exact_rep}. Since these formulas involve weighted volume integrals over \(x \in \R^d\), they can be approximated using the data samples in certain situations.

However, the obtained densities are still difficult to sample from since they involve integrals over the data space (which are not available, and thus can only crudely be approximated), and have to rely on rejection sampling on the high dimensional parameter space. Motivated by the corresponding exact formulas, we will instead suggest more simple approximations to the representation formula, which lead to simple densities that can efficiently be sampled from.
For sigmoidal functions, which we interpret as smoothed Heaviside functions, this will lead to a mixture density that randomly selects a data-point and samples for this data point \(x_k\) weights \(a,b\) associated to hyperplanes \(\{\,x \;|\; a \cdot x + b = 0\,\}\) that approximately fulfill:
\begin{itemize}
\item 
    orthogonality of the hyperplane to the gradient data; \(g_k \parallel a\), \(g_k \approx \nabla f(x_k)\),
\item 
    intersection with the data; \(a\cdot x_k + b \approx 0\),
\item
    relative likelihood of the hyperplane proportional to \(\norm{g_k}\).
\end{itemize}
The result of this can be seen in the middle plots of Figures~\ref{fig:intro_spatial_adapt} and~\ref{fig:intro_direct_adapt}.
Clearly, the outlined procedure delivers a marked improvement in accuracy over the uniform random feature model. Moreover, it is apparent that the nonuniform random weights are more clearly aligned with the optimal weights of the sparse, fully trained network.

Finally, we compare to another established technique of directional adaptivity, commonly referred to as active subspaces (AS)~\cite{ASM:2014,Constantine:2015,AS_VectorValued:2020} and applied to high dimensional ridge function neural networks in~\cite{DIPNet:2022}. Here, a truncated singular value decomposition of the gradient matrix
\[
\bld{U} \bld{\Sigma} \bld{V}
\approx \bld{G}
= \left[\begin{matrix}g_1 & g_2 & \ldots & g_K\end{matrix}\right]
\]
is used to obtain an optimal set of reduced variables \(x \approx \bld{U}_1 z\), with the active subspace spanned by the first \(d'\) columns of \(\bld{U}\) and \(z \in \R^{d'}\), \(d' < d\).
This can be applied to any regression technique in a pre-processing fashion. In the context of random weight initialization, we do not have to perform this pre-processing, and can simply sample the weights in a fashion that is proportional to the linear transformation \(\bld{U} \bld{\Sigma}\); see section~\ref{sec:AS}. While this works well for functions that are expressible on a reduced subspace, such as the two dimensional illustrative example from~\ref{fig:intro_direct_adapt}, our more general technique maintains benefits also if the reduced subspace changes in a nonlinear way as a function of the point \(x\). For an illustration see Figure~\ref{fig:intro_checkmark}, where a function that varies slowly along a nonlinear manifold in the form of a check-mark is approximated by different methods; see section~\ref{sec:example_checkmark} for a detailed description.
In the results, qualitative improvements are seen going from a random feature model with uniform weight selection (section~\ref{sec:uniform_random_features}) to a nonuniform random feature model (section~\ref{sec:gradient_nonuniform_random_features}) with the same number of weights.
As in the fully trained model, the hyperplanes induced by the inner weights align with the level set lines of the underlying functions, although there is no globally dominant direction along which the function is nearly constant.
\begin{figure}
\includegraphics[width=0.32\textwidth]{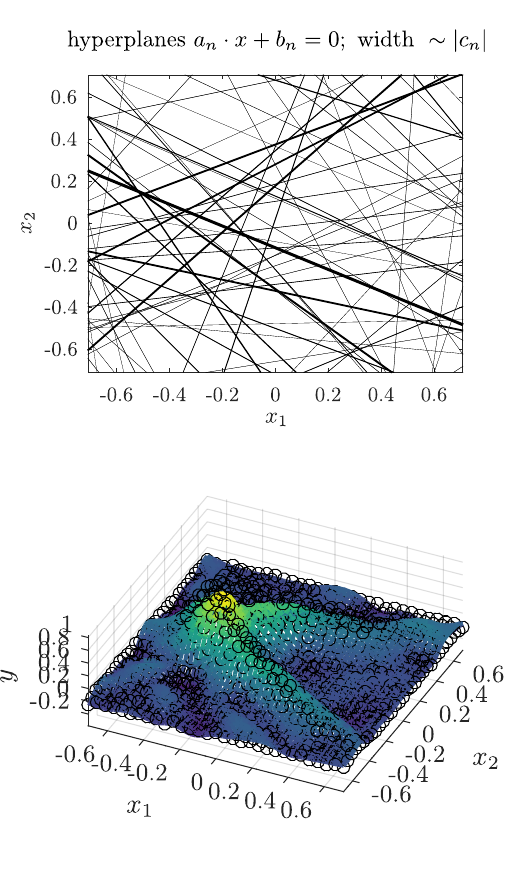}
\includegraphics[width=0.32\textwidth]{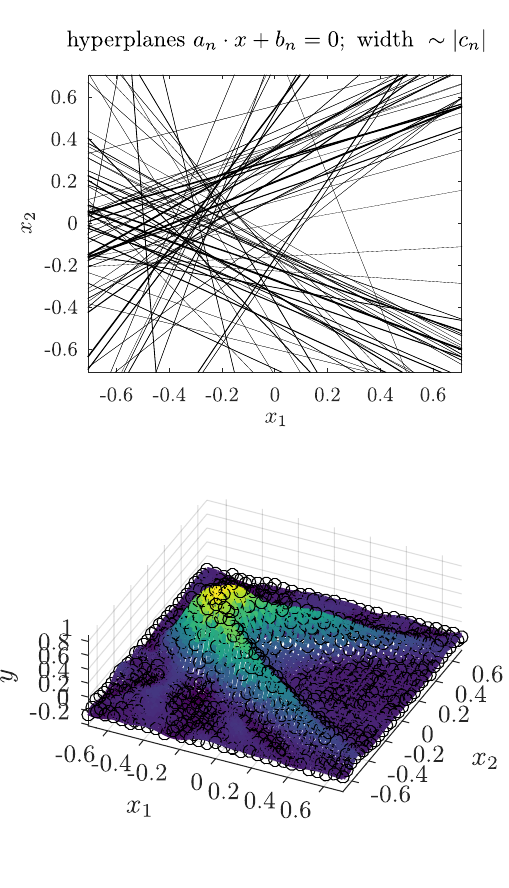}
\includegraphics[width=0.32\textwidth]{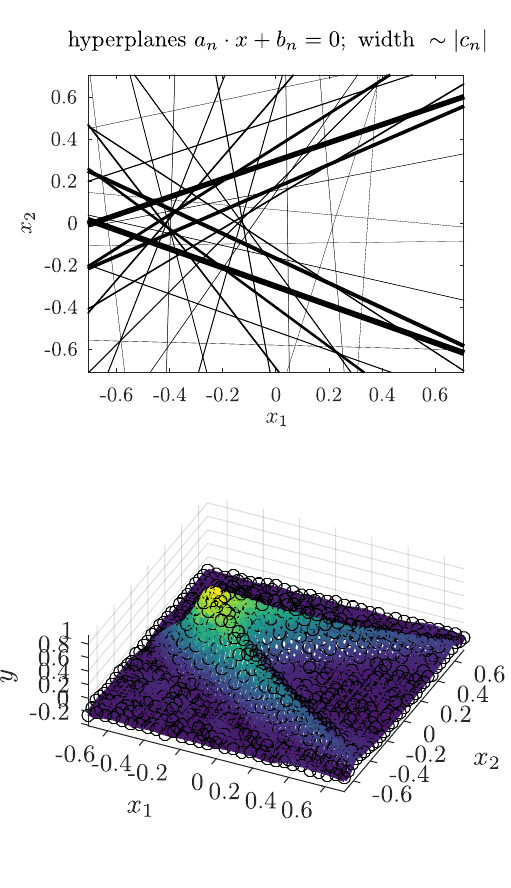}
\caption{Random feature models with a sigmoidal activation function, uniform weight sampling, gradient based weight sampling (each with \(N=75\)), and sparse feature learning with \(N=31\) (from left to right).
The data values are black circles and the model function is the surface plot (bottom); the inflection hyperplanes are visualized as lines (top).
The nonuniform random and sparse feature model are able to orient inflection hyperplanes with respect to directions of variability of \(f\) and associate them to regions in parameter space with high variability.
}
\label{fig:intro_checkmark}
\end{figure}

\subsection{Limitations and extensions}
\label{sec:discussion}

Although we consider smooth activation functions, we only treat them as smooth approximations to nonsmooth homogeneous activation functions; see section~\ref{sec:feature_functions}. By this choice, we fix the inner weight \(\norm{a} = 1\) and introduce a scaling hyper-parameter \(\delta > 0\) that represents the transition zone, where \(\delta=0\) is the infinitely fast transition representing a non-smooth homogeneous activation function.
This allows us to use the nonsmooth activation functions as a template for weight generation using the exact reconstruction formula, which is uniquely determined up to symmetries. However, it limits the flexibility of the resulting neural network approximation, which looses the multi-scale properties resulting from an individual choice of \(\norm{a} \sim 1/\delta\) for each neuron.
We note that the exact reconstruction formulas derived in~\cite{sonoda2017neural} for the general case contains a large degree of flexibility in the choice of the reconstruction formula (since the dual ridgelet is not uniquely determined).
Corresponding weight sampling formulas, as suggested in~\cite{NonparametricWeight:2013} based on the earlier results~\cite{murata1996integral} thus rely on a the specific choice of dual ridgelet and are not canonical, in contrast to nonsmooth activation functions that yield a unique reconstruction formula.
We expect that the proposed sampling strategies, which suggest important directions for the inner weights \((a,b)\) based on derivative information, will still be useful if \(\norm{a}\) or \(\delta\) is chosen differently for each neuron, but do not pursue this here.

Additionally, since we assume that gradient data \(g_k \approx \nabla f(x_k)\) is available, it would be natural to add it to the loss function (cf., e.g.~\cite{SobolevTraining:2017}) and consider a loss of the form
\begin{equation*}
  L(f; \bld{x}, \bld{y}, \bld{g})
  = \frac{1}{2K}\sum_{k=1}^K \left[(f(x_k) - y_k)^2
   + \lambda \norm{\nabla f(x_k) - g_k}^2\right],
\end{equation*}
for some \(\lambda > 0\). Although this would improve the accuracy of the trained model, it introduces additional complications concerning the regularity of the activation functions (since model gradients are not defined for Heaviside and defined only in a weak sense for ReLU), and the structure of the RKHS spaces associated to random feature models. For simplicity, we therefore restrict attention to loss function~\eqref{eq:empirical_loss}.

\section{Problem setting and notation}
\label{sec:problem_notation}

For simplicity of exposition we consider only the regression problem of learning a function $f\colon X \subset \R^d \to \R$ from values \(y_k \approx f(x_k)\) at number of input data points \(x_k\).
In general, we assume that the data is distributed according to some (potentially unknown) density \(\de \nu(x)\).
We also assume that the data have been centered and linearly transformed, such that most \(x_k\) tightly lie within a ball of radius \(R\) in the data space \(B_R(0) \subset \R^d\).
This is required to enable uniform sampling of the bias parameters \(b_k\); see Section~\ref{sec:uniform_random_features}.
In addition to the data values, we assume that we have access to
derivative values \(g_k \approx \nabla f(x_k) \in \R^d\).
There are many practical problems of interest, where these values are available, or can be easily generated by automatic differentiation at comparable cost to the data values \(y_k\).
For ReLU like activation functions (such as softplus) we also need Hessian values \(H_k \approx \nabla^2 f(x_k) \in \R^{d\times d}\).

We consider shallow neural networks in the form of the feature model
\begin{equation}
\label{eq:finite_network}
f(x)
\approx \NN(x;\bld{c},\bld{\omega})
= \sum_{n=1}^N c_n \varphi(x;\omega_n) + p_0(x),
\end{equation}
where \(p_0(x)\) is a simple polynomial of maximum degree \(s-1\) for \(s \in \{1,2,\ldots\}\). The parameters of the feature function are the inner weights \(\omega = (a, b) \in \R^{d+1}\) and every feature is given by a ridge function:
\[
\varphi(x;\omega_n)
= \sigma(a_n\cdot x + b_n),
\quad \sigma \colon \R \to \R.
\]
First, we introduce the activation functions \(\sigma\) given by the piecewise splines
\[
\sigma_s(t) = \max\{\,0,\; t\,\}^{s-1}/(s-1)!.
\]
In particular, for \(s=1\) we have the Heaviside function, and for \(s=2\) the rectified linear unit (ReLU).
Due to the homogeneity of the activation function, we can restrict the weights to the set \(\Omega\)~\eqref{eq:parameter_set};
see section~\ref{sec:feature_functions}.
Moreover, we also consider smooth versions given as
\[
\sigma_{s,\delta}(t) = [\sigma_s * \widehat{\eta}_\delta](t)
:= \int_{-\infty}^{\infty} \sigma_s(\tau) \widehat{\eta}_\delta(t - \tau)\de \tau,
\]
where \(\widehat{\eta}_\delta(t) = \widehat{\eta}(t/\delta)/\delta\) and \(\widehat{\eta}\) is an infinitely differentiable bump shaped function with integral one and vanishing first moment:
\begin{equation}
\label{eq:bump_function}
\widehat{\eta}(t) \geq 0,\quad
\widehat{\eta}(t) \to 0 \quad\text{for } t \to \infty,\quad
\int_{-\infty}^{\infty} \widehat{\eta}(t) \de t = 1,\quad \int_{-\infty}^{\infty} \widehat{\eta}(t) t \de t = 0.\quad
\end{equation}
We denote the network associated to this parameters as \(\NN_{\delta}\) or \(\NN\) if \(\delta = 0\).
We note that many popular activation functions fit in this framework, by an appropriate choice of the kernel \(\widehat{\eta}\).
Most notably, we consider:
\begin{enumerate}
\item 
    the classical sigmoid function
    \[
    \sigma_{1,\delta}(t) = (1 + \exp(-t/\delta))^{-1}
    = [\sigma_{1}*\widehat{\eta}_{\delta}] (t),
    \]
    is the convolution of the Heaviside function \(\sigma_1\) with the derivative of the sigmoid function
    \[
    \sigma_{1,\delta}'(t)
    = \widehat{\eta}_{\delta}(t)
    = \frac{1}{\delta} \widehat{\eta}\left(\frac{t}{\delta}\right)
    \quad\text{where }
    \widehat{\eta}(t)= (\exp(t/2) + \exp(-t/2))^{-2}
    = \operatorname{sech}(t/2)^2/4.
    \]
\item
    the softplus activation function is similarly a convolution with the ReLU (\(s=2\)):
    \[
    \sigma_{2,\delta}(t)
    = \delta \log(1 + \exp(t/\delta))
    = [\sigma_2 * \widehat{\eta}_{\delta}](t).
    \]
\end{enumerate}

The polynomial \(p_0\) has the maximum degree \(s-1\), and helps to simplify some derivations but is not strictly necessary. It is possible to write the polynomial in terms of a few specific ridge functions. For \(s \in \{1,2\}\) this is due to
\begin{equation}
\label{eq:ridge_to_polynomial}
(a\cdot x + b)^{s-1}/(s-1)! = \sigma_{s,\delta}(a\cdot x + b) + (-1)^{s-1} \sigma_{s,\delta}(-a\cdot x - b),
\end{equation}
which gives a constant one for \(s=1\), the linear function \(a\cdot x + b\) for \(s=2\), and serves as a basis for the \(s-1\) degree polynomial space by varying \(a\) and \(b\).
We thus omit the polynomial \(p_0\) in most cases, when it is not essential.

\subsection{Nonparametric regression model.}
To learn the functional relation \(f(x_k) \approx y_k\), we focus on standard least squares regression based on the minimization of the empirical loss~\eqref{eq:empirical_loss}
where \(f\) is given by a finite, but variable term approximation~\eqref{eq:finite_network} with undetermined \(N\) and weight matrices \(\bld{A},\bld{b},\bld{c}\). Since the possible number of parameters is not bounded, this model can best be understood as a finite discretization of an infinite (nonparametric) feature model, depending on how the weight initialization and training is performed.
The interpretation through this functional analytic lens allows to better understand the gap between random and sparse feature learning and how to construct nonuniform random feature models to narrow this gap.

\subsubsection{Infinite feature regression.}
\label{sec:infinite_feature_regression}

As a generalization of the finite, but variable term approximation~\eqref{eq:finite_network},
we consider the infinite width representation
\begin{equation}
\label{eq:network_measure}
\NN[\de\mu](x)
:= \int_\Omega \varphi(x;\omega)\de \mu(\omega),
\end{equation}
where \(\mu\) is a general distribution (i.e.\ a finite signed Radon measure) on the parameter set \(\Omega\) from~\eqref{eq:parameter_set}.
This generalizes the model~\eqref{eq:finite_network}, by the choice \(\mu = \sum_{n=1}^N c_n \delta_{\omega_n}\). We denote the space of finite Radon measures on the hidden parameter set by \(\M(\Omega)\), and introduce the non-parametric space
\[
\V 
= \{\,f = \NN[\de\mu]\;|\; \mu \in \M(\Omega) \,\}.
\]
Then, non-parametric regression according to~\eqref{eq:psi_regression} can be performed by minimizing the regularized loss
\begin{equation}
\label{eq:psi_regression_continuous}
\min_{\mu\in \M(\Omega)} L(\NN[\de\mu];\bld{x},\bld{y}) + \alpha\Psi(\mu),
\end{equation}
where \(\Psi\colon \M(\Omega) \to \R\) is the functional on the space of measures associated to \(\psi\) from~\eqref{eq:psi_regression} that measures the complexity of \(\mu\).
Classically, for \(\psi(t) = t\) this is the total variation norm \(\Psi(\mu) = \norm{\mu}_{\M}\)~\cite{bengio2006convex,RossetSwirszczSrebroZhu:2007,bach2017breaking,parhi2021neural}, but nonconvex variants with better sparsity properties have also recently been analyzed in the infinite context~\cite{Pieper2022}.
The problem~\eqref{eq:psi_regression_continuous} is also closely related to specific instances of classical neural network training with random initialization of weights and regularization by weight decay~\cite{Neyshabur:2015,parhi2021neural,chizat2018global,chizat2022sparse}.
Moreover, the training with sparsity promoting regularizers can be interpreted as producing optimal sparse approximations of exact representations as will be discussed in section~\ref{sec:exact_representation}.

\subsubsection{Random feature regression.}

Here, the weights will not be selected in an optimal way, adapted to the function \(f\), but in a random transferable way~\cite{zhang2023transnet}, that is able to represent any given function, given enough random samples. Here, the weights are sampled from a given distribution
\[
\omega_k \sim \de M(\omega),
\]
which can be interpreted as a random feature model (see, e.g.,~\cite{RahimiRecht:08,RudiRosasco:2017,bach2023relationship}) and the references therein. As a baseline, similar to~\cite{zhang2023transnet,bach2023relationship}, we consider the uniform distribution on \(\Omega_R\) from~\eqref{eq:parameter_set_R}.
If we denote the Lebesgue surface (Hausdorff) measure on the unit sphere by \(\de a\), and the Lebesgue measure on \(\R\) by \(\de b\), the uniform probability measure is given as
\begin{equation}
\label{eq:uniform_density_R}
\de M_R(a,b) = 1/m_R \de a \de b\rvert_{[-R,R]},
\end{equation}
where \(m_R = 2R \abs{\S^{d-1}}\) is the normalization constant
and \(\abs{\S^{d-1}} = 2 \pi^{n/2}/\Gamma(n/2)\) is the surface area of the sphere.

To analyze the neural network random feature model for a given probability measure \(\de M\) on \(\Omega\), one applies the law of large numbers to interpret the sum of random features as a Monte-Carlo approximation to the integral
\begin{equation}
\label{eq:integral_feature}
\NN[c\de M](x)
= \int_{\Omega} \varphi(x;\omega) c(\omega) \de M(\omega),
\end{equation}
with \(c_n = c(\omega_n) = c(a_n,b_n)\), where the coefficient function is in the Hilbert space \(c \in L^2(\Omega, \de M)\) associated to the prior distribution.

This is closely connected to the reproducing kernel Hilbert space induced by the kernel, where the finite sample is replaced with an expectation
\begin{equation}
\label{eq:kernel}
k_M(x,x')
:= \int_{\Omega} \varphi(x;\omega)\varphi(x';\omega) \de M(\omega)
\approx k_{M,N}(x,x').
\end{equation}
This kernel, in turn, is associate to the infinite/continuous problem~\cite{leroux07a}:
\begin{equation}
\label{eq:ridge_continuous}
\min_{c \in L^2(\Omega,\de M)}
L(\NN[c\de M];\bld{x},\bld{y})
+ \frac{\alpha}{2}\int_{\Omega} c(\omega)^2 \de M(\omega).
\end{equation}
In fact, by Monte-Carlo approximation of the integral network and the regularization term, we have
\begin{align*}
\NN[c \de M](x)
=
\int_\Omega c(\omega)\varphi(x;\omega)\de M(\omega)  
&\approx
\frac{1}{N}\sum_{n=1}^N c(\omega_n)\varphi(x;\omega_n),\\
\int_{\Omega} c(\omega)^2 \de M(\omega)
&\approx
\frac{1}{N}\sum_{n=1}^N c(\omega_n)^2.  
\end{align*}
where \(\omega_n\) are i.i.d.\ samples from the fixed distribution \(\de M\), and \(c\) is the density of the coefficient function with respect to the probability distribution. 
Substituting now \(c_n = N c (\omega_n)\), we obtain the finite version~\eqref{eq:ridge} from~\eqref{eq:ridge_continuous}.
The kernel and the parameterization \(\NN\) are intrinsically linked by the associated reproducing kernel Hilbert space (RKHS)
\begin{equation}
\label{eq:Hilbert_by_L2}
\H_M = \{\; f = \NN[c \de M] \;|\; c \in L^2(\Omega,\de M) \;\},
\end{equation}
endowed with the canonical norm and inner product. In fact, this is exactly the native Hilbert space associated to the kernel~\eqref{eq:kernel}.
This interpretation allows to quantify the discrepancy between the sparse feature network training~\eqref{eq:psi_regression_continuous}, and the random feature expansion~\eqref{eq:integral_feature}, by comparing the associated function spaces \(\V \hookrightarrow \H_M\):
for nonsmooth activation functions such as Heaviside or ReLU, the smoothness of the functions in either space is qualitatively different, with \(\H_M\) for the uniform measure \(M_R\) related to the Sobolev space \(H^{s+(d-1)/2}(B_R)\)
and \(\V\) containing functions of maximal regularity \(H^{s-1}(B_R)\) (e.g, a finite network).
In higher dimensions, the directional adaptivity of the variation space assigns much smaller norm to functions that are almost flat in some directions; see section~\ref{sec:uniform_random_features}.

\subsection{Nonuniform random feature models}
\label{sec:nonuniform_random_features}
We can now easily explain on a high level the difference in sparse infinite feature expansion and random feature expansion, based on the infinite versions associated to each.
Since any sampling probability density \(\de M\) and any coefficient function \(c \in L^2(\Omega,\de M)\) can be interpreted as the finite signed measure
\[
\de \mu = c(\omega)\de M(\omega) \in \M(\Omega),
\]
with Jensen's inequality
\[
\left(\int_\Omega \abs*{ c(\omega) } \de M(\omega)\right)^2
\leq
\int_\Omega c(\omega)^2 \de M(\omega)
\]
we obtain that
\[
\H_{M}
\hookrightarrow \V.
\]
Moreover, we have
\(\cup_{M \in M_1(\Omega), } \H_{M} = \V\),
which can be easily seen by showing that for any function \(f = \NN[\de \mu_f]\) the probability measure
\(\de M = \de \abs{\mu_f} / \norm{\mu_f}_{M(\Omega)}\)
with the non-negative total variation measure \(\abs{\mu_f} \in \M(\Omega)\)
induces a Hilbert space that contains a given \(f \in \V\).

The interpretation of this is straightforward: The degree to which random feature expansion is more restrictive than sparse infinite features is determined by how close \(M\) is to the normalized density of the total variation measure \(\abs{\mu_f}\).
If we can guess some information about the exact density, and incorporate that  into the distribution accordingly, we expect the random feature model to perform more like the infinite feature model.
In the following, we will use existing characterizations of these exact measures to derive such nonuniform densities.
For this purpose, we will require additional data on the derivatives of \(f\), which we will assume to have available.
In particular, for Heaviside-like activation functions \(\sigma_1\) and \(\sigma_{1,\delta}\) for small \(\delta > 0\), we need access to gradient data.

\section{Function spaces associated to neural network approximation}
\label{sec:function_spaces}

To characterize the optimal densities for non-uninform random feature models, we provide the necessary tools for precise analysis and characterization of neural network sparse representation and uniform random feature models.
We first provide some notation concerning the interpretation of smooth activation functions by nonsmooth ones. Then we provide results characterizing the densities of functions in the spaces introduced in the previous section that clearly show the gap between sparse and uniform random feature models, and form the basis for the nonuniform densities. We note that a precise characterization of these spaces is important since it forms the basis for approximation and generalization analysis; cf,~\cite{barron1994approximation,bach2017breaking,PaNo:2023,ReLUkApproximation:2022,ReLUkApproximation:2024,ReLUVariationApproximation:2024}.

\subsection{Activation functions.}
\label{sec:feature_functions}
We first consider a single neuron \(c'\sigma(a'\cdot x + b')\), and recall an alternative interpretation of the inner weights: using the hyperplane \(h_{a',b'} = \{\,x \;|\; a'\cdot x + b'\,\}\), and the reciprocal of the Euclidean norm of the weight vector \(\delta = \norm{a'}^{-1}\). Note that the hyperplane does not depend on \(\delta\) in the sense that we can replace \((a',b')\) by \((a,b) = \delta(a',b')\) without changing \(p\).
The activation functions \(\sigma_s\) fulfill the homogeneity property \(\sigma_s(t/\delta) = \delta^{1-s}\sigma_s(t)\) and thus
\[
c'\sigma_s(a' \cdot x + b')
= c'\sigma_s((a \cdot x + b)/\delta)
= c  \sigma_s(a \cdot x + b),
\]
where the scaling factor \(\delta\) has been absorbed into the new outer weight \(c = c'\delta^{1-s}\).
and thus it is sufficient to look for inner weights \(a\) on the unit sphere \(\S^{d-1} = \{\,a\in \R^d\;|\; \norm{a} = 1 \,\}\).
We follow the established convention of ignoring the offset parameter \(b\) for the normalization constant; see~\cite{savarese2019,ongie2019function,parhi2021neural}.
We note that there exists another convention, which is the normalization of the combined inner weight \((a,b)\) vector~\cite{bach2017breaking,Pieper2022}. The latter convention has the advantage of placing all inner weights on the compact set \(\S^d\), whereas the former convention usually leads to expressions with simpler interpretation and corresponds more closely with interpretations of shallow networks through the Radon transform~\cite{murata1996integral,sonoda2017neural}, which is why we prefer it here.

For the smooth versions, it the parameter \(\delta\) in the activation function \(\sigma_{s,\delta}\) is related to the parameter delta introduced above: it holds
\[
\sigma_{s,1}(t/\delta) = 
\int \sigma_s(t/\delta - t') \widehat{\eta}(t') \de t'
= \int \sigma_s((t - \tilde{t})/\delta) \widehat{\eta}( \tilde{t}/\delta)/\delta \de \tilde{t}
=
\delta^{1-s}  [\sigma_s * \eta_{\delta}](t)
= \delta^{1-s}  \sigma_{s,\delta}(t),
\]
and thus
\[
c'\sigma_{s,1}(a' \cdot x + b')
= c \sigma_{s,\delta}(a \cdot x + b).
\]
Thus, if \(\delta\) is left as a learnable parameter, we can also restrict the inner weight \(a\) to the unit sphere.
However, in this work, we consider a simplified case, where \(\delta>0\) is chosen as a fixed parameter for all neurons ahead of time.
We note that this restricts the flexibility of the associated neural network, which in the general case can learn the norm of the inner weight \(\norm{a}\), which corresponds, as we have seen, to the smoothing parameter \(\delta\) for each neuron. However, it simplifies the following arguments, which are based on interpreting \(\delta\) as a relatively small smoothing parameter of the non-smooth activation \(\sigma_s\) (tied essentially to the data resolution, the spacing of the data points \(x_k\)). And we can use the known characterization results for \(\delta = 0\). Clearly we expect all of the the developed methods to improve (in a similar way) by lifting this requirement, but postpone this to future work.
Concerning the offset, we can restrict attention to the set \(b \in [-R,R]\), (or \(b \in [-R-\delta,R+\delta]\)) since for \(a \in \S^{d-1}\) hyperplanes with \(b\) outside of that interval do not intersect the data set \(X \subset B_R(0)\).

\subsection{Sparse feature models}
\label{sec:sparse_convex}

We consider the sparse measure representation~\eqref{eq:network_measure}.
Then we define the variation norm
\[
\norm{f}_{\V}
= \inf \{\,\norm{\mu}_{\M(\Omega)}
\;|\; \NN[\de\mu] = f \; \text{on } X\,\}
\]
associated to the space \(\V\).
For \(\delta = 0\), this norm is related to a Barron-norm, established in~\cite{Barron:1993,breiman1993hinging,klusowski2018approximation}. For more recent results and refinements, we refer to~\cite{ma2022barron,ReLUkVariation:2023}.
It contains all functions for which there exists an extension \(f\) to \(\R^d\) such that
\[
\abs{f}_{\mathcal{B}_s}
=
\int_{\R^d} \abs{\widetilde{f}(\xi)} \abs{\xi}^s \de \xi
\]
is finite, where \(\widetilde{f} = \Four[f]\) is the Fourier transform of \(f\), i.e., we have \(\V \hookrightarrow \mathcal{B}_s\).
Note that the problem~\eqref{eq:psi_regression_continuous}
for \(\Psi(\mu) = \norm{\mu}_{\M(\Omega)}\)
and the sparse regression problem
\begin{equation}
\label{eq:psi_regression_variation}
\min_{f \in \V} L(f; \bld{x}, \bld{y}) + \alpha \norm{f}_{\V},
\end{equation}
and its finite dimensional approximation~\eqref{eq:psi_regression} for \(\psi(c) = \abs{c}\) are equivalent, by a representer theorem; see, e.g.~\cite{Bredies:2019,Representer:2019}.
Finally, for more precise characterizations of the variation space for homogeneous nonsmooth activation functions with \(\delta=0\) we refer to~\cite{ongie2019function,PaNo:2023} for the ReLU case and to~\cite{parhi2021neural,ReLUkVariation:2023,RKBS:2023} for \(s\geq 2\).

\subsubsection{Characterization of functions.}
\label{sec:exact_representation}
The functions in the space \(\V\) can be further characterized, and exact representation formulas
\[
f(x) = \NN[\de\mu_f](x) = \int_{\R^d}  \varphi(x;\omega) \de \mu_f(\omega)
\]
with can be given for regular functions, including all infinitely smooth compactly supported functions.
For this, we introduce the Radon transform
\[
\Rad[f](a,b) = \int_{a\cdot x + b = 0} f(x) \de x
\]
and its (formal) adjoint
\[
\Rad^*[\chi](x) = 
\int_{\mathbb{S}^{d-1}} \chi(a,-a\cdot x) \de a .
\]
Note that with respect to the common parameterization of the hyperplane \(a\cdot x = p\), we have chosen the parameter \(b=-p\), since then
\[
\sigma(a\cdot x + b) = \Rad^*[\widehat{\sigma}_{a,b}](x),
\quad\text{ where } \widehat{\sigma}_{a,b}(a',b')
= \sigma(b - b') \delta_{a}.
\]
We note that the above formulas are only suggestive on a formal level and their rigorous functional analytic interpretation requires extending the Radon transform and its adjoint to spaces of distributions; see, e.g., \cite{ongie2019function,parhi2021neural,RKBS:2023}.
This is done by extension by duality of the inner product on \(L^2(\R^d)\) resp.\ \(L^2(\Omega)\)  with \(\Omega =\S^{d-1} \times \R\).
Moreover, we note that Radon transforms of functions are even, i.e.\ \(\mathcal{R}[f](a,b) = \mathcal{R}[f](-a,-b)\), and odd functions lie in the kernel of its adjoint.
On Radon space, we also introduce the differential \(\partial_b\) in the offset variable and we note that \((-\partial_b)^s \widehat{\sigma}_{a,b}(a',b') = \delta_{a,b}\) holds for \(\sigma = \sigma_s\).
To invert the adjoint Radon transform, 
we also introduce the operator \(\Lambda_b\) as the square root of the one dimensional negative fractional Laplacian
\[
\Lambda_b = (-\partial^2_{b})^{1/2}
= \Four_b^{-1} (\abs{\xi} \Four_b (\cdot)),
\]
where \(\Four_b\) is the one-dimensional isometric Fourier transform with respect to \(b\)
and the Radon inversion formulas~\cite{helgason2011integral} for even rapidly decaying functions on the Radon domain
\begin{equation}
\label{eq:radon_inversion}
c_d \Lambda_b^{d-1} \Rad \Rad^* \chi = \chi
\end{equation}
with \(1/c_d = 2 (2 \pi)^{d-1}\).
Thus, the inverse of the Radon transform and its adjoint are \(\Rad^{-1} = c_d \Rad^* \Lambda_b^{d-1}\) and \(\Rad^{-*} = c_d \Lambda_b^{d-1} \Rad\) (for even functions on the Radon domain).

\subsubsection{Characterization of networks:}
First, consider the case \(\delta = 0\) and recall the following result.

\begin{theorem}[\cite{ongie2019function,parhi2021neural,RKBS:2023}]
\label{thm:exact_rep}
For sufficiently smooth \(f \in \V\) (e.g., \(f \in C_c^{d+s}(\R^d)\)) we define:
\begin{equation}
\label{eq:representation}
c_{f}(a,b) = c_d (-\partial_b)^{s} \Lambda^{d-1}_b \Rad[f](a,b)
= (-\partial_b)^{s}  \Rad^{-*}[f](a,b),
\end{equation}
which is an even function for \(s\) even, and an odd function for \(s\) odd. Together with a polynomial \(p_{f}\) of maximal degree \(s-1\) the function \(f\) has the unique decomposition
\[
f(x) = p_{f}(x) + \NN[c_{f}(a,b) \de a \de b](x),
\]
where the polynomial and the coefficient function are uniquely determined (in the space of even or odd measures). 
\end{theorem}
\begin{remark}
The polynomial can be written as \(p_{f}(x) = \NN[c_{p}(a,b)\de a \de b]\) for an infinite number of odd (for \(s\) even) or even (for \(s\) odd) functions. Moreover, the representation formula can also be extended to all functions in \(\V\), by extending the involved operators to appropriate spaces of distributions.
\end{remark}

For the smooth versions, we recall that the Radon transform of radial functions is independent of the variable \(a\) and therefore there exists a unique radial \(\eta_\delta(x)\) in the data space associated to \(\widehat{\eta}_\delta[b]\) through the inverse Radon transform, i.e.,
\begin{equation}
\label{eq:radial_mollifier}
\widehat{\eta}_\delta(b) = \Rad[\eta_\delta](b).
\end{equation}
In fact, taking the adjoint Radon transform and using the Radon inversion formula~\eqref{eq:radon_inversion} we have
\[
\eta_\delta
= c_d \Rad^*[ \Lambda_b^{d-1} \widehat{\eta}_\delta].
\]
Then, for a smoothed version of \(f\) we
use the commutation of the Radon transform and its adjoint with convolutions, i.e.
\[
\Rad [\eta_\delta * f]
= \Rad [\eta_\delta] *_b \Rad [f]
= \widehat{\eta}_\delta *_b \Rad [f]
= \int_{\R^d} \widehat{\eta}_{\delta}(a\cdot x + b) f(x) \de x.
\]
This simple insight will not allow us to derive an exact representation for a given \(f\) using the network \(\NN_{\delta}\), which is not our goal. However, it will easily allow us to exactly represent a smooth approximation of \(f\).

\begin{corollary}\label{cor:exact_rep}
For sufficiently smooth \(f \in \V\) as in Theorem~\ref{thm:exact_rep}, we have
\[
f_\delta(x) :
 = \eta_\delta * f
 = p_{f,\delta}(x) + \NN_{\delta}[c_{f}(a,b) \de a \de b](x),
\]
where the polynomial and the coefficient function are uniquely determined (in the space of even or odd measures). Moreover, the same function can be exactly represented with the network \(\NN\) using a smoothed density
\[
\NN_{\delta}[c_{f,s}(a,b) \de a \de b](x)
=\NN[ c_{f,\delta}(a,b) \de a \de b](x)
\quad\text{where } c_{f,\delta} = \widehat{\eta}_\delta *_b c_{f}.
\]
The smoothed density \(c_{f,\delta}\) has the direct representation
\begin{equation}
\label{eq:representation_delta}
c_{f,\delta}(a,b)
= \int_{\R^d} \widehat{\psi}_{s,\delta}(a\cdot x + b) f(x) \de x,
\quad\text{where } \widehat{\psi}_{s,\delta} = c_d (-\partial_b^{s}) \Lambda_b^{d-1} \widehat{\eta}_\delta.
\end{equation}
Finally the doubly smoothed function \(f_{\delta,2} = \eta_\delta *\eta_\delta * f\) is exactly represented with the density \(c_{f,\delta}\) and the network \(\NN_{\delta}\).
\end{corollary}

This corollary only allows to represent smooth approximations (blurred versions) of the function \(f\) to be approximated, using the fact that a radial blur filter in spatial domain corresponds to a blur filter in the Radon domain. Our subsequent approach to suggesting exact densities is based on the assumption that \(\delta\) is chosen sufficiently small such that \(f_\delta\) is close enough to \(f\), (e.g., up to the resolution of the available training data). However, this blurring also smooths out any discontinuities in the function \(f\) and its representation measure \(c_{f,\delta}\), allowing us to represent it as a volume integral~\eqref{eq:representation_delta} instead of with a hyper-surface integral as in~\eqref{eq:representation}, which is more difficult to approximate.

\subsection{RKHS for random feature models}
\label{sec:uniform_random_features}
The Hilbert space~\(\H\) introduced in~\eqref{eq:Hilbert_by_L2} is exactly the native Hilbert space associated to the kernel \(k_M\) from section~\eqref{eq:kernel}; see, e.g.,~\cite[Section~10]{Wendland_2004}.
It has the representation through the kernel by definition
\[
\norm{f}_{\H_M}
= \sup \left\{\; \pair{f,g} \;\Big|\;  \norm{g}_{\H_M'} \leq 1 \;\right\},
\]
where
\[
\norm{g}^2_{\H_M'}
= \int_{\R^d} \int_{\R^d} k(x,x') g(x) g(x') \de x \de x'.
\]
For the given activation functions, this can be further characterized, and we introduce the spaces \(\H_{R,s}\) and \(\H_{R,s,\delta}\) associated to the activation functions and distribution above.

For the smooth version with \(\delta > 0\) we simply use that
\(\NN_{\delta} = \eta_\delta * \NN\) to obtain:
\begin{proposition}
\label{prop:delta_Hilber_char}
The space \(\H_{M,s,\delta}\) can be characterized by the equality
\begin{equation}
\label{eq:delta_Hilbert_char}
\H_{M,s,\delta}
= \{\;\eta_\delta * f \;|\; f \in \H_{M,s} \;\},
\end{equation}
where \(\H_{M,s}\) is the space defined for \(\delta=0\).
The associated kernel is given by \(k_{M,s,\delta}(x,x') = [\eta_\delta *_x k_{M,s} *_{x'} \eta_\delta](x,x')\), where  \(k_{M,s}\) is the kernel associated to \(\delta = 0\).
\end{proposition}
This shows that whatever regularity is imposed on \(f\) though its inclusion in the Hilbert space associated to \(\delta = 0\), is only improved by an additional convolution with the radial mollifier \(\eta_\delta\) introduced in~\eqref{eq:radial_mollifier}.

\subsubsection{Kernel associated to uniform distribution:}
For the distribution \(M_{R}\) and ridge splines with \(\delta = 0\) the kernel is given as
\[
k_{R,s}(x,x')
= \frac{1}{2R \abs{\mathbb{S}^{d-1}}}\int_{\mathbb{S}^{d-1}} 
\int_{-R}^R \sigma_s(a\cdot x + b)\sigma_s(a\cdot x' + b) \de b \de a .
\]
Explicit expressions for this kernel were derived in~\cite{bach2023relationship} extending the classical case \(s = 1\), which is the Heaviside activation function~\cite[Section~3.3]{leroux07a}. For \(\norm{x}\) and \(\norm{x'}\) smaller than \(R\), we have
\begin{equation}
\label{eq:kernel_HS}
k_{R, s}(x,x')
= P_{R,s}(x,x') - c_{R,s}\norm{x-x'}^{1+2(s-1)},
\end{equation}
with a constant \(c_{R,s}>0\) and a (positive-semidefinite) polynomial kernel \(P_{R,s}(\norm{x}^2,\norm{x'}^2,\norm{x-x'}^2)\) involving the \((s-1)\)-degree polynomial \(P_{R,s}\), which are not of further interest in the following.
This kernel is closely related to the kernel \(k(x,x') = -\norm{x-x'}^{1+2(s-1)}\), a conditionally positive definite kernel related to polyharmonic splines; see, e.g., \cite[Section~8, Corollary~8.18]{Wendland_2004}.

The associated native Hilbert space is related to the (fractional) Sobolev space \(H^{(d-1)/2 + s}(\R^d)\),
which can be endowed with the semi-norm
\[
\abs{f}^2_{H^{(d-1)/2 + s}}
=
\int_{\R^d} \abs{\widetilde{f}(\xi)}^2 \abs{\xi}^{d-1 + 2s} \de \xi
\]
where \(\widetilde{f} = \Four[f]\) is the Fourier transform of \(f\). Note, that in the case of odd \(d\), this is the Beppo-Levi-space with square-integrable partial derivatives of order \((d-1)/2 + s\); cf.~\cite[Section~10.5]{Wendland_2004}.
However, the representation~\eqref{eq:kernel_HS} is only valid for \(x, x' \in B_R\). Since we assume that all data points are contained in this ball \(x \in X \subset B_R\), this is enough to obtain a characterization of the space.
In \cite[Proposition~4]{bach2023relationship} the following result can be found, which is valid for all \(s\) and based on the explicit characterization of the kernel~\eqref{eq:kernel_HS}.
\begin{proposition}
We have the norm equivalence
\[
\H_{R,s}
\sim
H^{(d-1)/2 +s}(B_R)
= \{\, f'\rvert_{B_R} \;|\; f' \in H^{(d-1)/2 + s}(\R^d) \,\},
\]
which is the restriction of (fractional) Sobolev functions to the ball of radius \(R\) endowed with the norm
\[
\norm{f}^2_{H^{(d+1)/2 + s}(B_R)}
= \inf \left\{
c'\norm{f'}_{L^2(B_R)}^2
+
\abs{f'}^2_{H^{(d-1)/2 + s}}
\;\Big|\; f'\rvert_{B_R} = f
\right\},
\]
for some suitable \(c'>0\).
\end{proposition}

For \(\delta > 0\) we recall Proposition~\ref{prop:delta_Hilber_char}, and note that this is a strict subspace of this space of inifinitely smooth functions, with norm that is similar to the Sobolev norm for slowly varying functions (varying more slowly than \(\delta\)) and much larger for more radidly varying functions.

\section{Gradient based nonuniform random weight selection}
\label{sec:gradient_nonuniform_random_features}

The question we ask is how to efficiently sample candidate weights
\[
\omega_n = (a_n,b_n) \sim \de M_{\bld{x},\bld{y},\bld{d}}(\omega)
= m_{\bld{x},\bld{y},\bld{d}}(\omega) \de \omega
\]
from a data dependent nonuniform distribution \(m_{\bld{x},\bld{y},\bld{d}}\) that allows for good approximation of the function \(f\) with few samples.
Motivated by Section~\ref{sec:nonuniform_random_features} and Theorem~\ref{thm:exact_rep}, for functions that possess an  exact density \(c_f(a,b)\de a \de b\), the normalized density is an excellent candidate for the sampling distribution
\[
m_{\bld{x},\bld{y},\bld{d}}(a,b)
\approx \frac{1}{m_f} \abs{c_f(a,b)} ,
\quad\text{where } m_f = \int \abs{c_f(a,b)}\de a\de b
\]
However, neither analytic expressions for \(f\) nor \(\mu_f\) are available at the stage of generating the candidate weights. Moreover, for functions that are not exactly representable or do not have a continuous density \(\de \mu_f(a,b) = c_f(a,b)\de a \de b\), this density might be difficult to sample from.
Thus, in the following, we identify simplified approximations of \(c_\delta(a,b)\) based on smoothed versions \(f_\delta\) of \(f\) and corresponding approximate densities (Corollary~\ref{cor:exact_rep}) to sample inner weights. 

\subsection{Approximate densities based on exact representation}
\label{sec:exact_sampling}

We are interested to employ derivative information of the function \(f\) at the data points, and thus we are looking for exact representation formulas that reveal its influence on the exact density. This will be particularly relevant for Heaviside and sigmoidal networks.

We define the averaged gradient function
\[
\widetilde{g}_{s,\delta}(a,b)
= \int_{\R^d} \nabla f(x) \widehat{\psi}_{s-1,\delta}(a\cdot x + b) \de x,
\quad\text{for } (a,b) \in \S^{d-1}\times \R,
\]
which aggregates the gradient of \(f\) over a neighborhood of the hyperplane \(a\cdot x + b = 0\) using an influence factor \(\widehat{\psi}_{s-1,\delta}\).
It might seem intuitive to choose \(\widehat{\psi}_\delta\) to be a bump function around zero, but, motivated by the exact reconstruction formula from Theorem~\ref{thm:exact_rep},
we select a weight function with the following properties:
\[
\widehat{\psi}_{s-1,\delta} \in C_0(\R), \quad
\int_\R \widehat{\psi}_{s-1,\delta}(b) b^k \de b
=
\begin{cases} 
0 &\text{for } 0 \leq k < d + s - 2,\\
1 &\text{for } k = d + s - 2.
\end{cases}
\]
In particular, we choose
\[
\widehat{\psi}_{s-1,\delta}
= c_d  (-\partial_b)^{s-1} \Lambda^{d-1}_b\widehat{\eta}_{\delta}.
\]

\begin{proposition}
\label{prop:exact_integral_density}
Let \(f \in C^{d+s}_c(X)\) for a domain \(X\subset \R^d\).  Then, the exact density 
of the smoothed function \(f_\delta = \eta_\delta * f = \NN[c_\delta]\) is given by
\[
c_{\delta}(a,b)
= a \cdot \widetilde{g}_{\delta}(a,b)
= \int_{\R^d} a\cdot \nabla f(x) \widehat{\psi}_{s-1,\delta}(a\cdot x + b) \de x.
\]
\end{proposition}
\begin{proof}
From the exact representation Theorem~\ref{thm:exact_rep} and Corollary~\ref{cor:exact_rep} we have
\[
c_{\delta}(a,b)
= c_d (-\partial_b)^{s} \Lambda^{d-1}_b \Rad[\eta_\delta * f](a,b)
= \int_{\R^d} \widehat{\psi}_{s,\delta}(a\cdot x + b) f(x) \de x.
\]
Now we realize that due to \(a \in \S^{d-1}\) and thus \(a\cdot a =1 \) we have
\[
\widehat{\psi}_{s,\delta}(a\cdot x + b)
= -\partial_b \widehat{\psi}_{s-1,\delta}(a\cdot x + b)
= -a \cdot a \partial_b \widehat{\psi}_{s-1,\delta}(a\cdot x + b)
= -a \cdot \nabla_x \left[\widehat{\psi}_{s-1,\delta}(a\cdot x + b)\right]
\]
and therefore
\[
c_{\delta}(a,b)
= - a \cdot \int_{\R^d} \nabla_x \left[\widehat{\psi}_{s-1,\delta}(a\cdot x + b)\right] f(x) \de x
= a \cdot \int_{\R^d} \widehat{\psi}_{s-1,\delta}(a\cdot x + b) \nabla f(x) \de x,
\]
using integration by parts.
\end{proof}

\begin{remark}
Although this result is valid for any \(s\), in a sense it is most instructive for Heaviside/sigmoidal networks with \(s=1\). 
In this case the averaged gradient is simply the smoothed inverse adjoint Radon transform of the gradient
\[
\widetilde{g}_{1,\delta}
= \widehat{\eta}_\delta *_b \Rad^{-*}[\nabla f]
= \Rad^{-*}[\eta_\delta * \nabla f].
\]
\end{remark}

Based on this result, we can suggest a strategy to sample weights according to the unnormalized density
\[
m_{s,\delta}(a,b) = \abs{a\cdot \widetilde{g}_{s,\delta}(a,b)},
\]
which measures how well aligned the averaged gradient over a neighborhood of the hyperplane \(a\cdot x + b\) is to the normal vector \(a\). In practice, we still have to replace the integral over \(\R^d\) by some finite approximation. This is possible only in cases where:
\begin{enumerate}
\item
The exact distribution \(\nu\) associated to the data samples is known and has density \(\de\nu(x) = \rho(x) \de x\),
\item
The function \(f\) is compactly supported on the domain of interest \(X\subset \R^d\) (or, fulfills \(\nabla f(x) \approx 0\) for \(x \not \in X\)),
\item
An upper bound \(m_{s,\delta}(a,b) \leq \bar{m}\) is available.
\end{enumerate}
In that situation, the expression above can be given by the Monte Carlo approximation
\begin{align}
\notag
m_{s,\delta}(a,b)
&= \abs*{\int_X a \cdot \nabla f(x) \frac{\widehat{\psi}_{s-1,\delta}(a\cdot x + b)}{\rho(x)} \, \rho(x)\de x} \\
\label{eq:approximate_density_MC}
&\approx \abs*{ \frac{1}{K} \sum_{k=1}^K a \cdot g_k \frac{\widehat{\psi}_{s-1,\delta}(a\cdot x_k + b)}{\rho(x_k)}}
= m_{s,\delta,K}(a,b),
\end{align}
where we recall that \(g_k \approx \nabla f(x_k)\) is the gradient data.
To obtain samples from this density we can then resort to rejection sampling based on samples from the uniform density from section~\ref{sec:uniform_random_features} and
the acceptance ratio \(0 < m_{s,\delta,K}(a,b) / \bar{m} < 1\).
We refer to~\cite{NonparametricWeight:2013} for a similar approach in the context of bump-shaped activation functions (corresponding to a choice of \(s = 0\), which we do not consider, since the case for \(\delta=0\) does not relate to any activation function, but rather the Dirac delta; see~\cite{sonoda2017neural}). We note that that the requirement 3.\ could be removed by using more sophisticated techniques as rejection sampling techniques such as Markov chain Monte Carlo methods such as the Metropolis-Hastings algorithm; see, e.g.~\cite{neal2012bayesian}.

\subsection{Simplified approximations for sampling densities}
\label{sec:approximate_sampling}

The method based on approximate exact representation has two major drawbacks that make it unattractive in practice, which are caused not by the form of the density, but by how we approximate it in~\eqref{eq:approximate_density_MC} and how we obtain samples from it. 

First, to rewrite the integral as a sum, we need to know that the exact density of the data samples \(\rho\), which usually would not be available (unless the samples happen to be uniform in some set and \(\rho \equiv \text{const}\), for instance).
We could simply replace \(\rho\) by a constant, but this would introduce another approximation error. Second, the formula assumes that \(f\) (its gradient) is approximately compactly supported, which is not the case for functions as~\eqref{eq:example_1d2d} that are constant in certain directions, despite having a simple representation in terms of a shallow network that is accurate on the region of interest. We will see that applying the density~\eqref{eq:approximate_density_MC} to such functions leads to artifacts in the approximate density caused by truncating the integral to the data domain.
In general, we aim to construct nonuniform densities \(\de M(a,b)\) with the properties that
\begin{equation}
\label{eq:support_condition_networks}
\supp M(a,b)
\subseteq \{\, (a, b) \;|\; a \in \operatorname{range} \nabla f,\, b \in -a \cdot X_\delta \,\}.
\end{equation}
The first condition ensures that if \(f\) is constant in direction \(d\), i.e., \(d\cdot \nabla f \equiv 0\), we will only generate weights \(a\) with \(d\cdot a = 0\), such that the corresponding network has the same property.
The second property ensures that no offset parameters associated with hyperplanes not intersecting a neighborhood of the support of the data \(X_\delta\) are generated, e.g., \(X_\delta = \{\,x'\;|\; x \in X, \norm{x-x'} \leq \delta\,\}\).

Second, the use of rejection sampling comes with potentially big drawbacks. We have seen already in section~\ref{sec:intro_illustrate} that optimal densities can concentrate on small subsets of the possible parameter set \(\S^{d-1} \times \R\), which leads to a very large rejection probability, if the weights are first uniformly sampled. We would prefer to have a more efficient way of sampling weights.
Possible strategies motivated by the reconstruction formula will be discussed in the following.
We present the two most commonly used types of activation function first \(s \in \{\,1,2\,\}\) and do not present the derivations for activation functions corresponding to larger \(s\), which are not often used in practice and would require higher order derivatives.

\subsection{Heaviside and sigmoidal networks}
First, we consider the case of a sigmoidal activation function~\(\sigma_{1,\delta}\).
For \(\delta > 0\), the gradient of the finite network is given simply as:
\begin{equation}
\label{eq:grad_network_1}
\nabla f(x)
\approx \nabla \NN_{\delta}(x;\bld{A},\bld{b},\bld{c})
= \sum_{n=1}^N c_n \widehat{\eta}_\delta(a_n\cdot x + b_n) \, a_n,
\end{equation}
where \(\widehat{\eta}_\delta(a_n\cdot x + b_n)\) is the bump function introduced in~\eqref{eq:bump_function}.
The fundamental idea of the simplified densities is that the gradient of the network is given by a linear combination of the inner weights \(a_n\), weighted by coefficients related to a discrete version of the adjoint Radon transform.
Instead of inverting that transform, we simply build candidates for the weights directly based on the gradient data.

The construction will lead to a method that can be interpreted as a localized version of a global method associated to ``active subspaces'', which we briefly outline first.

\subsubsection{Method of active subspaces.}
\label{sec:AS}
A simple way to ensure the property~\ref{eq:support_condition_networks} is based on the method of active subspaces~\cite{ASM:2014,constantine2015computing}.
We consider the discrete data case, and simply sample an average combination of all gradients multiplied by i.i.d.\ normal distributed coefficient and normalize it to the sphere:
\[
\widetilde{g}
= \sum_{k=1}^K \xi_k g_k, \quad\text{where } g_k = \nabla f(x_k), \;\xi_k \text{ i.i.d.\ } \mathcal{N}(0,1).
\]
By the properties of Gaussian random variables, we can perform a singular value decomposition of the gradient matrix
\[
G = [g_k]_{k=1,\ldots,K} = V \Sigma W,
\quad V \in \R^{d\times r}, \Sigma \in \R^{r\times r}, V \in \R^{r\times K},
\]
where \(r = \rank K\) is the rank of \(G\),
and generate the same distribution by
\[
\widetilde{g} = V \xi
\quad\text{where } \xi \text{ is } \mathcal{N}(0,\Sigma).
\]
Alternatively, \(\widetilde{g}\) is a Gaussian random variable with expectation zero and covariance 
\[
C = G G^\top = V\Sigma^2 V^\top.
\]

Then, we generate weights by projecting these random vectors to the sphere
\[
a = \frac{\widetilde{g}}{\norm{\widetilde{g}}_2} \in \S^{d-1}
\]
and a corresponding offset can be chosen uniformly from \(b \in -a \cdot D\), which are all \(b\) such that the hyperplane \(a\cdot x + b = 0\) intersects the data support.
This density is given by the parametric form
\[
\de M^{AS}_{G}(a,b)
= 
\frac{\chi_{b\in[-R,R]}}{2R} \de b \;
m_{C}(a) \de a ,
\]
where the density of \(a\) on the sphere \(\S^{d-1}\) is a projected normal distribution at location zero, also known as the ``angular central Gaussian'' distribution~\cite{ACG:1987}:
\begin{align*}
m_C(a)
&= \int_{0}^\infty \tau^{d-1}\sqrt{\frac{\det(C^{-1})}{(2\pi)^{d}}}
\exp(-\tau^2 a^\top C^{-1} a/2)
\de \tau \\
&=
\frac{1}{\abs{\S^{d-1}}}
\sqrt{\frac{\det(C^{-1})}{\left({a^\top C^{-1} a}\right)^{d}}}
\quad\text{where } \abs{\S^{d-1} } = \frac{2 \pi^{{d}/{2}}}{\Gamma({d}/{2})}.
\end{align*}
This density is only valid if \(\rank G = \rank C = d\); otherwise the density is supported only the \(r\) dimensional sphere with \(r = \rank C\) given by \(\S^r = \S^{d-1} \cap \range C\). On that sphere the density has the same form 
\begin{align}
\label{eq:density_sphere}
m_C(a)
=
\frac{1}{\abs{\S^{r-1}}}
\sqrt{\frac{\det_+(C^{+})}{\left({a^\top C^{+} a}\right)^{r}}},
\quad a \in \S^{r-1} = \S^{d-1} \cap \range C,
\end{align}
where \(C^+\) is the pseudo-inverse of \(C\) and \(\det_+\) the pseudo determinant (the product of the nonzero eigenvalues).
In this method, we sample the bias uniformly \(b \sim \mathcal{U}(-R,R)\), as in the uniform method.

\subsubsection{Local gradient based sampling.}
\label{sec:gradient_sampling}
However, the active subspace approach averages gradient directions over the whole domain, and does not distinguish between different locations \(x\) that lie close to the hyperplane \(a\cdot x + b = 0\).
Alternatively, we construct a density on the hyperplanes as follows: We start by randomly selecting a data point \(x_k\) with associated gradient \(g_k \approx \nabla f(x_k)\) according to the magnitude of the gradient \(\norm{g_k}\), motivated by formula~\eqref{eq:grad_network_1}.
Then we sample a hyperplane \((\hat{a}_k, \hat{b}_k)\) with \(\hat{a}_k = \pm g_k / \norm{g_k}_2\) and \(\hat{b}_k = - \hat{a}_k\cdot x_k\).
This is the hyperplane parallel to the gradient of \(f\) at \(x_k\) that intersects the point \(x_k\). We introduce the function \(w\colon \R^d\times \R^d\setminus\{0\} \to \R^{d+1}\) with \(w(x,g) = (g, -x\cdot g) / \norm{g}_2\), which allows to write \((\hat{a}_k, \hat{b}_k) = \pm w(x_k,g_k)\).

This leads, in the data discrete case, to the discrete distribution
\begin{equation}
\label{eq:gradient_density_discrete}
M^0_{G}
= \frac{1}{m_K} \sum_{k=1}^K \frac{\norm{g_k}_2}{2} \left(\delta_{w(x_k, g_k)} + \delta_{-w(x_k, g_k)}\right),
\end{equation}
where \(m_K = \sum_k \norm{g_k}_2\) is the normalization factor to obtain a probability distribution (we also exclude points with \(g_k = 0\) from the sum). We could also add the data density \(\rho(x_k)\), as in~\eqref{eq:approximate_density_MC}, but since it is usually not known, we leave it out immediately. This density is very easy to sample from, since is a discrete distribution with finite support,
and there exist very efficient methods to sample from categorical distributions (using a uniform \(\mathcal{U}([0,1])\) random variable, and the cumulative probabilities).
However, this is also a big drawback since we only obtain a finite set of feature functions with this approach.
It is thus not clear that this distribution will guarantee a universal approximation property.

Note that in the continuous data case, the corresponding continuous probability distribution on the parameter space can be given as the average of the push-forwards of the measure \(\norm{\nabla f(x)}_2 \de \nu (x)\) with respect to the mappings \(x \mapsto w(x,\nabla f(x))\) and \(x \mapsto -w(x,\nabla f(x))\):
\[
\int \chi(a,b) \de M_{\nabla f}(a,b)
= \frac{1}{m_{\nabla f}} \int \left(\chi(w(x,\nabla f(x)) + \chi(-w(x,\nabla f(x))\right) \norm{\nabla f(x)}_2 \de \nu (x),
\]
using that \(\nabla f = 0\) only on a \(\norm{\nabla f}_2 \de \nu\)-zero set.

\subsubsection{Nonlocal gradient based sampling.}
\label{sec:gradient_nonlocal}

To obtain a distribution with continuous density, we allow for uncertainty in the weights and offsets.
For this purpose, we introduce a second hyperparameter \(\delta_W\), that will determine a approximate radius in the physical space around each data point \(x_k\) from where gradients of the function \(f\) will be aggregated to determine a distribution on the space of parameters (hyperplanes).

In this more general formulation, for each data point \(x_k\) we first define a weight~\(w_{k',k}\) that determines how much the data from another point \(x_{k'}\) will influence the sampling of hyperplanes associated to \(x_k\). We require \(w_{k',k} = w_{k,k'} > 0\). For instance, we can specify the form of the weights \(w_{k,k'}\) with a simple radial kernel with approximate width \(\delta_W\) as
\[
w_{k,k'} = \exp(-\norm{x_k - x_{k'}}_2/(2\delta_W)).
\]
Then a ``nonlocal gradient'' random variable is defined:
\[
\widetilde{g}_k
= \sum_{k'=1}^K w_{k,k'} \xi_{k'} g_{k'},
\quad\text{where } \xi_{k'} \sim \mathcal{N}(0,1) \text{ i.i.d. }
\]
Together with the weight matrix \(W = [w_{k,k'}]_{k,k'=1,\ldots K}\), and \(\widetilde{g}_k = G \diag(W_k) \xi\), this is a centered Gaussian random variable with the covariance
\[
C_k
= \sum_{k'} w_{k,k'}^2 g_{k'}^{\phantom{\top}} g_{k'}^\top
= G \diag(W_k^2) G^\top,
\]
where \(W_k\) is the \(k\)-th row of the weight matrix.
It is clear that this choice generalizes both the first and the second method. The active subspace approach is obtained by using \(w_{k,k'} = 1\) for all \(k,k'\), which means that all \(C_k\) are identical (as is obtained by \(\delta_W \gg R\)).
The second method above takes \(w_{k,k'} = \delta_{k,k'}\) (as results from \(\delta_W \to 0\)), and thus \(\widetilde{g}_k\) is proportional to \(g_k\).

Now that we have determined a way to sample a direction, it remains to turn it into a distribution on the space of hyperplanes. First, we obviously normalize \(a_k = \widetilde{g}_k / \norm{\widetilde{g}_k}\), which induces the density \(m_{C_k}(a)\) given in~\eqref{eq:density_sphere}.
after having obtained \(a\), we sample an offset \(b\) from a density that is centered around the offset that makes the corresponding hyperplane intersect \(x_k\):
\[
m_{x_k}(a,b) = \bar{\eta}_{\delta_W}(a\cdot x_k + b),
\]
where \(\bar{\eta}_{\delta_W}\) is a centered distribution with standard deviation \(\sim \delta_W\), e.g., a Gaussian \(\mathcal{N}(0, \delta_W^2)\).
Then, we build the density as a mixture of theses single densities over all data points \(k\):
\begin{equation}
\label{eq:gradient_nonlocal_density_discrete}
\de M^W_{G}(a,b)
= \frac{1}{m^W_K}\sum_{k=1}^K \sqrt{\tr(C_k)} \; m_{x_k}(a,b) \de b \; m_{C_k}(a) \de a\rvert_{\range C_k},
\end{equation}
where \(m^W_K = \sum_{k=1}^K \sqrt{\tr(C_k)}\).

The choice of the exponential kernel here is essentially arbitrary, any kernel of approximate width \(\delta_W\) should yield similar results.
We note that if these weights are strictly positive (as is the case for the radial formula above), all ranges \(\range C_k\) are equal to the range of the gradient field \(\range G\). Thus, the support of the density~\eqref{eq:gradient_nonlocal_density_discrete} contains the full parameter set \(\Omega_R\) intersected with the range of \(G\), similar to the active subspace distribution from section~\ref{sec:AS}. Thus, the universal approximation property of the approximation scheme for uniform sampling method is preserved for all functions \(f\) that are constant on the complement of \(\range G\).
Thus, unless the range of the gradients \(g_k\) in the samples \(x_k\) is not equal to the range of the gradient over all \(x \in X\) (which is very unlikely if \(K \gg d\)), the function \(f\) is constant on \(\range G\) and can be approximated arbitrarily well by using an increasing number of parameters sampled from~\eqref{eq:gradient_nonlocal_density_discrete}.

\subsubsection{Sampling from the approximate densities.}

The constructed densities are mixture models based on mixtures over the data points and associated simple projected Gaussian and and one-dimensional distributions, for which efficient sampling methods exist. We first randomly select a number of data points according to a categorical distribution (using a histogram of uniformly \(\mathcal{U}([0,1])\) values with break points given by the cumulative probabilities), and then directly sample the associated weight \((a_n,b_n)\) based on its underlying (angular central) Gaussian distribution.

\subsection{ReLU and softplus networks}
For ReLU networks, an alternative version of Proposition~\ref{prop:exact_integral_density} can be given as follows. Here, the important quantity is given by an inverse adjoint Radon transform of the Hessian instead of the gradient, so we build simplified densities from this quantity.
\begin{proposition}
\label{prop:exact_integral_density2}
Let \(f \in C^{d+s}_c(X)\) for a domain \(X\subset \R^d\).  Then, the exact density 
of the smoothed function \(f_\delta = \eta_\delta * f = \NN[c_\delta]\) is given by
\[
c_{\delta}(a,b)
= a^T \widetilde{H}_{s,\delta}(a,b) a
= \int_{\R^d} a^T \nabla^2 f(x) a \, \widehat{\psi}_{s-2,\delta}(a\cdot x + b) \de x.
\]
\end{proposition}
Again, the result is valid for all \(s\), but for \(s = 2\) due to \(\widehat{\psi}_{s-2,\delta} = \Lambda_b^{d-1} \widehat{\eta}_\delta\) has an interpretation as a smoothed inverse adjoint Radon transform of the Hessian
\[
\widetilde{H}_{2,\delta}
= \widehat{\eta}_\delta *_b \Rad^{-*}[\nabla^2 f]
= \Rad^{-*}[\eta_\delta * \nabla^2 f].
\]
To derive direct expressions for efficient sampling, we compute the Hessian
\[
\nabla^2 f
\approx \nabla^2\NN_{\delta}(x;\bld{A},\bld{b},\bld{c})
= \sum_{n=1}^N c_n \widehat{\eta}_\delta(a_n\cdot x + b_n) a_n a_n^T,
\]
using that the polynomial \(p_0\) is a linear polynomial and fulfills \(-\nabla^2 p_0 = 0\). 
Thus each data point \(x_k\) with Hessian \(\nabla^2 f(x_k) = H_k\) suggests sampling neuron with \(\hat{a}\) from the span of the Hessian, \(\hat{b} \approx - \hat{a}\cdot x_k\) and likelihood related to the eigenvalues of the Hessian.

Analogously to the gradient based sampling, we define a ``nonlocal Hessian vector'' random variable as
\[
\widetilde{h}_k
= \sum_{k'=1}^K w_{k,k'}  H_{k'}\xi_{k'},
\quad\text{where } \xi_{k'} \sim \mathcal{N}(0,I_d) \text{ i.i.d. }
\]
This is again a centered Gaussian random variable with the covariance
\[
C_k
= \sum_{k'} w_{k,k'}^2 H_{k'}^{\phantom{\top}} H_{k'}^\top.
\]
Based on these definitions, the Hessian based density
\(M^W_{H}\) can be defined just as in~\eqref{eq:gradient_nonlocal_density_discrete}. In this case, even for \(w_{k,k'} = \delta_{k,k'}\) a continuous density is obtained (unless the function \(f\) is linear in certain directions), but a small nonlocality in the weights \(w\) appears reasonable to improve robustness.

\subsection{Weight sampling based on residual.}
\label{sec:nonuniform_residual}

The local methods discussed in sections~\ref{sec:approximate_sampling} are only rough approximations suggested by the exact formula from section~\ref{sec:exact_sampling}. Thus, there can be situations where they perform in a sub-optimal way, by oversampling regions of the parameter space with large derivatives of the function. Although this is largely mitigated in our experiments by the more robust versions from section~\ref{sec:gradient_nonlocal} (see section~\ref{sec:numerics}), we briefly mention an incremental residual based approach that also performs well.

Consider the case \(s=1\) and suppose we have sampled a small number of weights \(\omega^{(0)}_n\) \(n = 1,\ldots,N_0\) from a distribution \(\de M_{\bld{x},\bld{g}^{(0)}}\) for the gradient vector \(\bld{g}^{(0)}\) associated to \(f\) and trained outer weights \(c^{(0)}_n\) by ridge regression~\eqref{eq:ridge}. The gradient of the remaining residual is given by
\[
g^{(1)} = \nabla_x r^{(0)}
= \nabla_x ( f - \NN(x; \bld{\omega}^{(0)}, \bld{c}^{(0)}) )
= g^{(0)} - \nabla_x \NN(x; \bld{\omega}^{(0)}, \bld{c}^{(0)}).
\]
Since the second expression is only defined for \(\delta > 0\), we generally assume this for this subsection.

Since the shallow network representation~\eqref{eq:shallow_intro} is additive, we can thus consider the new task of approximation of the residual with a network. To sample more inner weights \(\omega_n\), \(n = N_0+1,\ldots, N_1\), we thus employ the density \(\de M_{\bld{x},\bld{g}^{(0)}}\) based on the gradients of the residual. Then, we compute new outer weights \(\bld{c}^{(1)}\) with~\eqref{eq:ridge} for all sampled weights.
By spacing the numbers of neurons in a logarithmic way \(N_i = \lceil \kappa^i N_0 \rceil\) with \(\kappa > 1\), we can ensure that the computational cost for ridge regression and regularization parameter choice for \(\alpha\) is dominated by the cost of the last problem with the largest number of parameters. However, we caution that this method is sensitive to the correct choice of the hyper-parameters \(\alpha\) and \(\delta\), since we rely on the regularization and smoothing to enable generalization and stable derivatives through the associated spaces from sections~\ref{sec:function_spaces}. If either of them is selected too small, taking a gradient of the network is not stable, the gradients of the residual can be a large even in areas where the residual is already small.

\section{Numerical examples}
\label{sec:numerics}

To compare the proposed sampling strategies, we perform several tests. The code repository to reproduce these results can be found at~\cite{NonuniformRandomFeatureCode}. In particular, we demonstrate that the proposed adaptive methods consistently outperform uniform sampling and simpler global strategies such as active subspaces. The difference in performance is determined by the structure of the underlying function. In all cases, the optimal regularization parameter is selected with the following cross-validation procedure: Additional to the input data we generate another validation data set from the same distribution (of equal size) together with corresponding output data. The weights are trained only on the training data, and the error is evaluated on the training plus validation data. Then, the regularization parameter is chosen by a grid search as the largest parameter that comes within \(5\%\) of the smallest validation error observed.
In general, we consider \(\delta > 0\) on the order of resolution of the data, and for the nonlocal strategies from section~\ref{sec:gradient_nonlocal} we employ \(\delta_W = 2\delta\).

As another comparison, we also consider sparse feature learning, using the nonconvex regularization strategy from section~\ref{sec:infinite_feature_regression}. 
Here, the number of weights is chosen adaptively based on the regularization parameter, and thus we simply solve the problem for a variety of regularization parameters.
Compared to the well-established convex approach summarized in section~\ref{sec:sparse_convex}, this leads to sparse representations with lower number of neurons \(N\), while maintaining similar accuracy. We do not expect the nonuniform sampling methods to outperform the sparse optimization methods in terms of the number of neurons required for accurate representation. Instead, we consider it to be the best currently possible training method, which is based on gradient based training and neuron insertion and deletion (as occurs in boosting, vertex exchange, or generalized conditional gradient methods). Clearly, this approach is orders of magnitude more expensive during training than a random feature model (for the same number of neurons), and thus random feature models have a speed advantage during training. However, during evaluation of the network, only the number of neurons determines the computational effort, and thus finding the smallest network with the best possible performance during training is beneficial. To carry over this advantage to a random feature model it is important that the nonuniform random weights come as close as possible to the fully trained weights.

\subsection{Instructive test-cases}

We start with some structurally simple examples to highlight important behaviors of different methods.

\subsubsection{Different activation functions in one dimension}
\label{sec:example_1d}

To illustrate the basic concepts, we consider first the simple one dimensional example from the introduction: \(f(x) = \exp(- (10x)^2/2)\), using \(1000\) uniformly spaces points \(x_k\) in the unit interval, and additive Gaussian noise and \(\delta = 1/80\).
Since the setup fulfills all requirements of all nonuniform densities described in section~\ref{sec:gradient_nonuniform_random_features}, we compare in particular the nonuniform density from sections~\ref{sec:exact_sampling} and~\ref{sec:gradient_sampling}. The former is based on a standard approximation (mollification) of an exact representation, and can be considered close to ideal. The second is the most simple to evaluate (out of all presented options), but may not be optimal for all situations. However, for this trivial example, both are not significantly different in practice; see Figure~\ref{fig:convergence_gauss}.
\begin{figure}
    \includegraphics[width=0.45\textwidth]{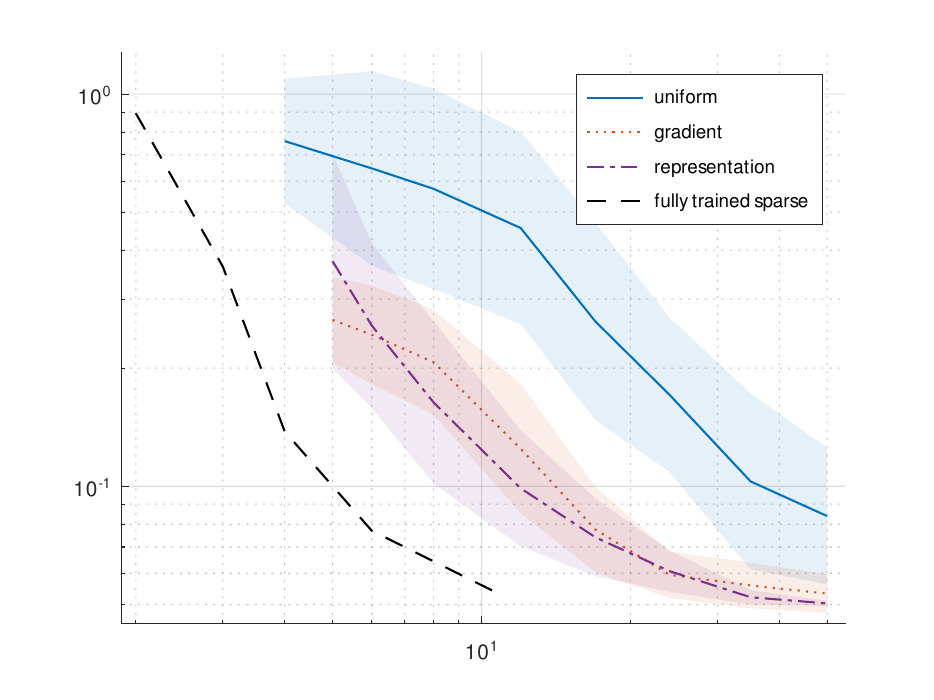}
    \includegraphics[width=0.45\textwidth]{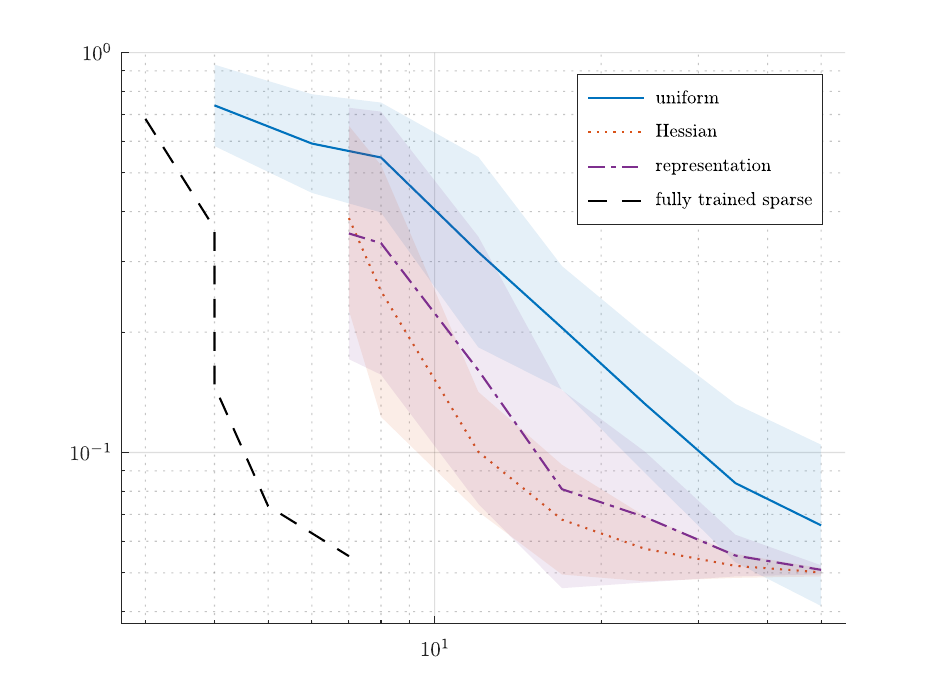}
    \caption{Convergence plot for the Gaussian bump, left: \(s=1\), right: \(s=2\).}
    \label{fig:convergence_gauss}
\end{figure}
However, we observe a significant difference between uniform sampling, nonuniform sampling and optimally trained networks. The quantitative error over the number of parameters shows that the fully trained networks learns the function (up to the noise level) with \(N=11\), and nonuniform sampling around \(N=30\), as illustrated in the introduction in Figure~\ref{fig:intro_spatial_adapt}. However, the uniform sampling has not reached this accuracy at \(N=50\).

\subsubsection{Instructive examples in two dimensions.}
\label{sec:example_2d}
Next, we consider several examples in two dimensions. Here, we use \(1000\) uniformly spaces points \(x_k\) in a centered cube with side length \(\sqrt{2}\), which tightly fits inside the unit circle, together with \(\delta = 1/40\) and \(s=1\).
We compare the different nonuniform sampling strategies based on the gradient from section~\ref{sec:approximate_sampling}, including the active subspace based approach and the residual based version together with the reconstruction formula based density from section~\ref{sec:exact_sampling}.

\subsubsection{Planar wave in two dimensions.}
First, we consider the second example from the introduction
\(f(x) = \sin(5 (x_1 + \sqrt{2}x_2))\).
We observe that the gradient based method performs significantly better than uniform sampling, and comes close to the optimal sparse approach, as already illustrated qualitatively in Figure~\ref{fig:intro_direct_adapt}. However, for this example, also the active subspace based anisotropic uniform sampling method from section~\ref{sec:AS} performs very well, since it also correctly identifies the global subspace that the function varies on. However, the approximation of the exact representation based formula from section~\ref{sec:exact_sampling} does not perform well, since the function does not fulfill the requirements of the derivation (notably it does not have compact support).
\begin{figure}
    \includegraphics[width=0.45\textwidth]{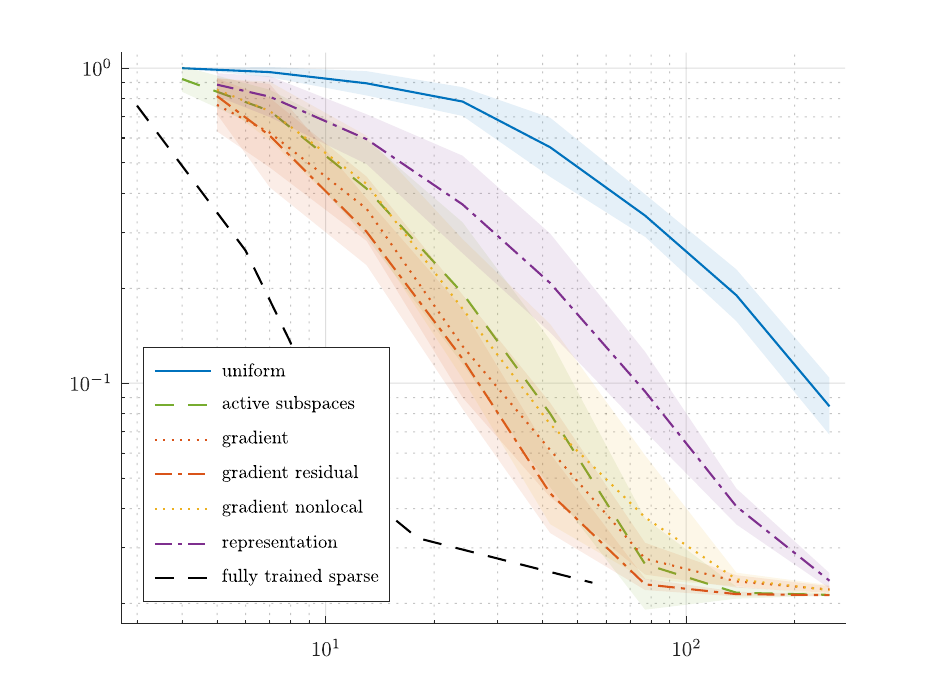}
    \includegraphics[width=0.45\textwidth]{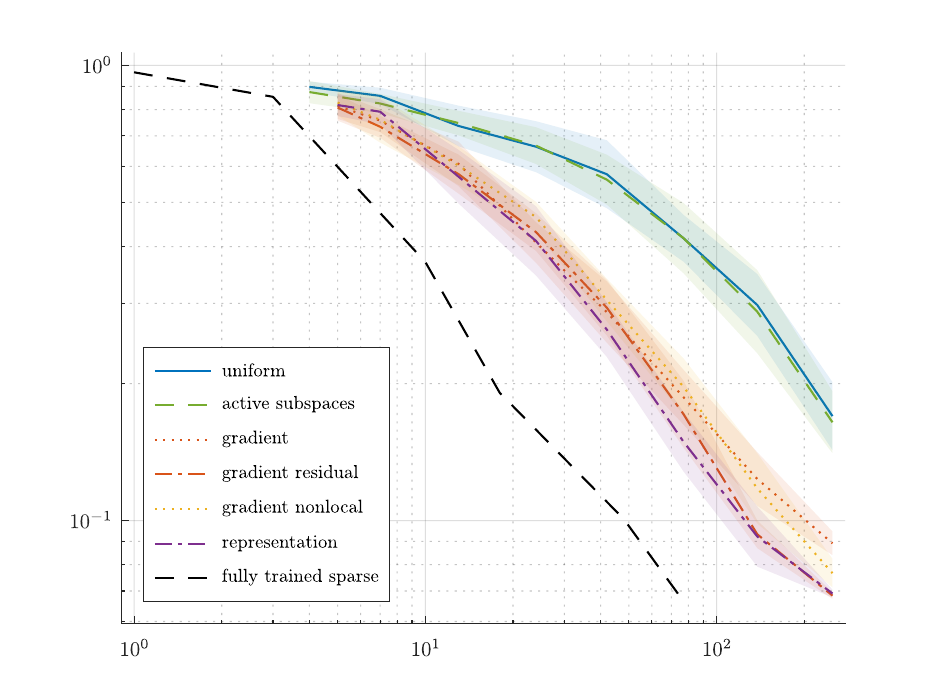}
    \caption{Convergence plots for anisotropic examples from the introduction; left: globally anisotropic planar wave, right: locally anisotropic Gaussian ``checkmark'' function.}
    \label{fig:convergence_2d}
\end{figure}

To further illustrate this, we plot the densities \(\de M\) associated to: the discretization of the reconstruction formula~\eqref{eq:approximate_density_MC}, the gradient based density~\eqref{eq:gradient_density_discrete}, and the sparse approximation to the exact density learned with sparse feature learning~\ref{eq:psi_regression_continuous}.
We visualize these densities in the following way: first, we obtain a large number of samples for the random models, and interpret them as equally weighted samples. Then, we obtain a distribution suitable for visualization purposes by fitting a kernel density estimate (KDE) with a kernel on \(\S^{1} \times [-1,1]\) with empirically chosen width, where the sphere is parameterized in terms of its angle. For sparse feature learning, we fit a KDE to the measure \(\sum_n \abs{c_n} \delta_{(a_n,b_n)}\), where \((a_n,b_n,c_n)\) are the trained weights. 
The results are given in the top row of Figure~\ref{fig:density_2d}. Since the gradient based samples and the learned weights perfectly align with the active direction \((1, \sqrt{2})\), the corresponding KDE estimates are supported on a small neighborhood of the line representing this direction (and its negative) up to the width of the KDE estimate.
However, the formula based on exact representation is affected by a substantial amount of smearing in the angular direction of \(a\), owing to the fact that the underlying formula is only valid for compactly supported functions, which is not the case here, where a global function was truncated to a finite square domain.

\begin{figure}
    \includegraphics[width=0.32\textwidth]{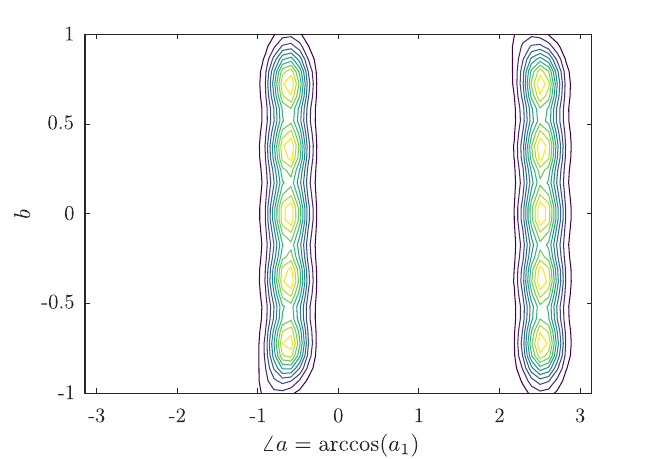}
    \includegraphics[width=0.32\textwidth]{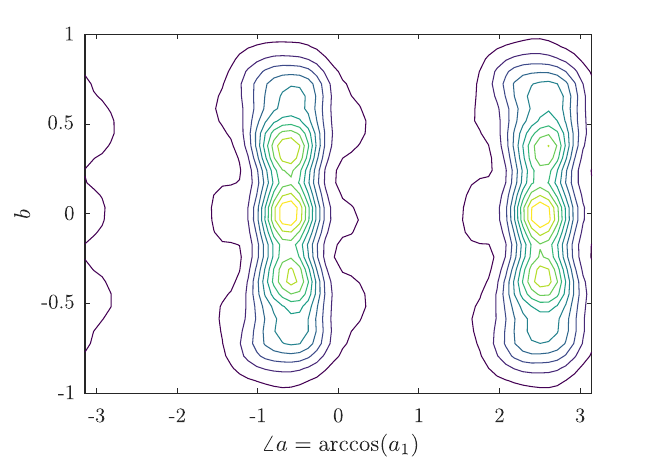}
    \includegraphics[width=0.32\textwidth]{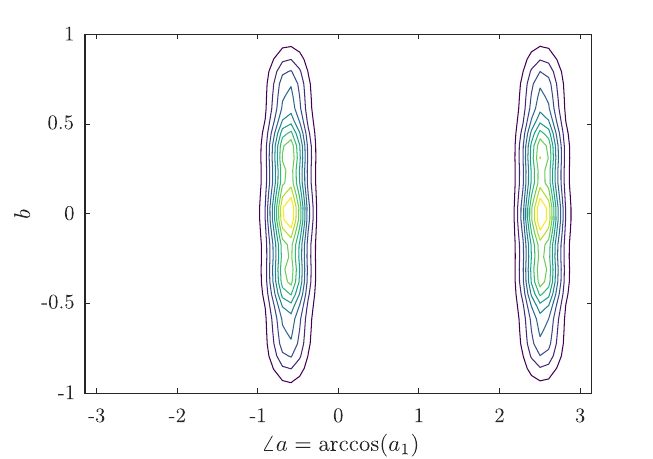}
    \includegraphics[width=0.32\textwidth]{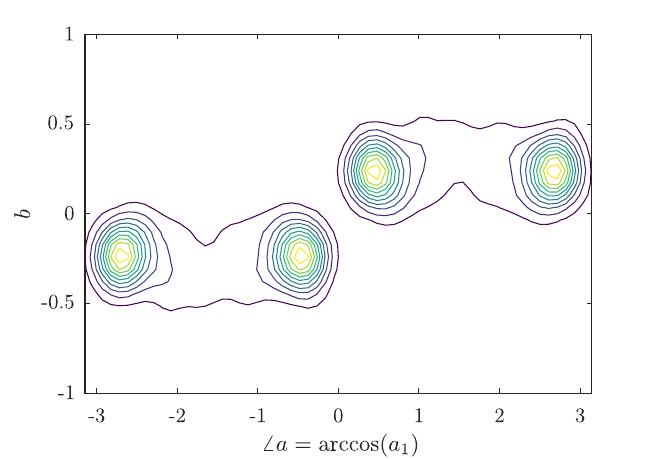}
    \includegraphics[width=0.32\textwidth]{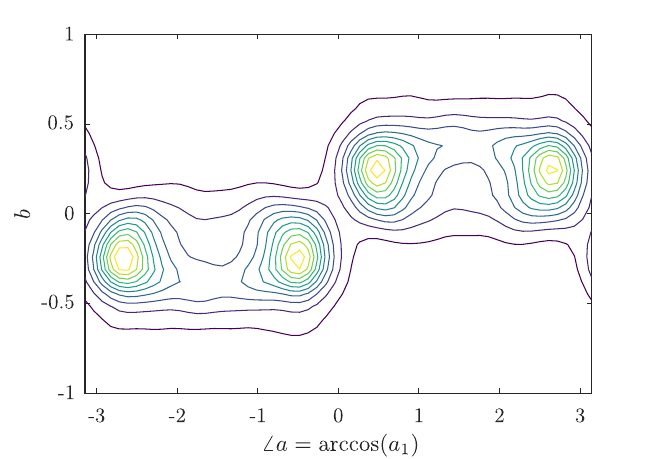}
    \includegraphics[width=0.32\textwidth]{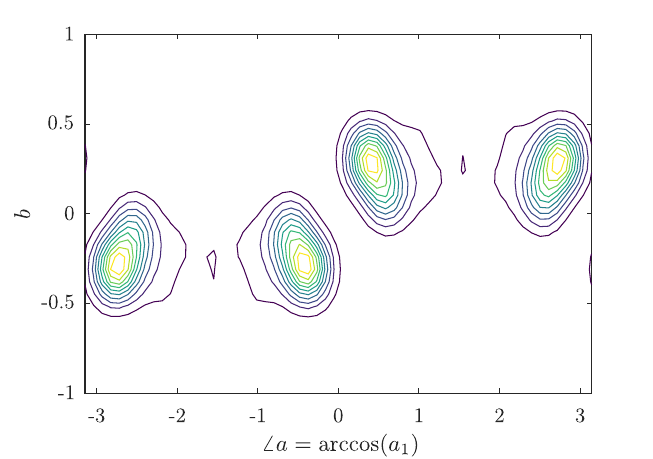}
    \caption{Density of different distributions on parameter space (KDE estimate); left to right: density from sparse feature learning, integral density, and gradient based sampling; top: globally anisotropic planar wave, bottom: locally anisotropic ``checkmark'' function.}
    \label{fig:density_2d}
\end{figure}

\subsubsection{Nonlinear manifold Gaussian function}
\label{sec:example_checkmark}
Now, we turn to the nonlinear locally anisotropic example from the introduction:
It is given in \(d\) dimensions as
\[
    f(x) = \exp( - \norm{\Sigma^{-1/2}T(x)}_2^2 / 2),
\]
where \(\Sigma^{-1/2} = \diag(\sigma_i) = \diag(8,4,2,1,1/2,\dots)\) and
\[
    T(x) 
    = \left[x_1 - H(x_2,x_3,\ldots,x_d), x_2, x_3 \ldots, x_d\right]^T,
\quad
    H(x_2,x_3,\ldots,x_d)
    = - \frac{1}{3} + \frac{2}{3} 
    h(3 \norm{x}_2),
\]
where \(h(t) = t^2\) for \(\abs{t} < 1/2\) and \(h(t) = \abs{t} - 1/4\) for \(\abs{t} > 1/2\) is the Huber function.

In two dimensions, this function represents a density that is concentrated around the graph of the function \(H\) in a neighborhood of \(x_2 = 0\), thus resembling a check-mark; see Figure~\ref{fig:intro_checkmark}.
We first give the quantitative results for the images shown in the introduction:
the approximation quality in two dimensions over the number of neurons is given in Figure~\ref{fig:convergence_2d}. Here, the assumptions on the support are fulfilled (approximately, the function decays rapidly), and all nonuniform sampling densities based perform similarly, with the exception of the globally anisotropic sampling inspired by active subspaces, which performs similar to global uniform sampling. This is explained by the fact that this function does not have any global anisotropy that can be exploited.

To see if this behavior is reflected in higher dimensions, we also test \(d=3\) and \(d=4\) on the centered cube with side length \(2/\sqrt{d}\) and \(K=2000\) random training samples (and validation samples). Results for gradient based sampling (and nonlocal and residual variants) are given in Figure~\ref{fig:convergence_checkmark_hd}. As before, the active subspace distribution performs the same as uniform sampling, whereas nonlocal and residual based nonuniform sampling perform significantly better, closer to sparse feature learning. However, it is also visible that the simple gradient based strategy deteriorates for higher accuracy, likely due to the finite set of directions.
\begin{figure}
\includegraphics[width=0.45\textwidth]{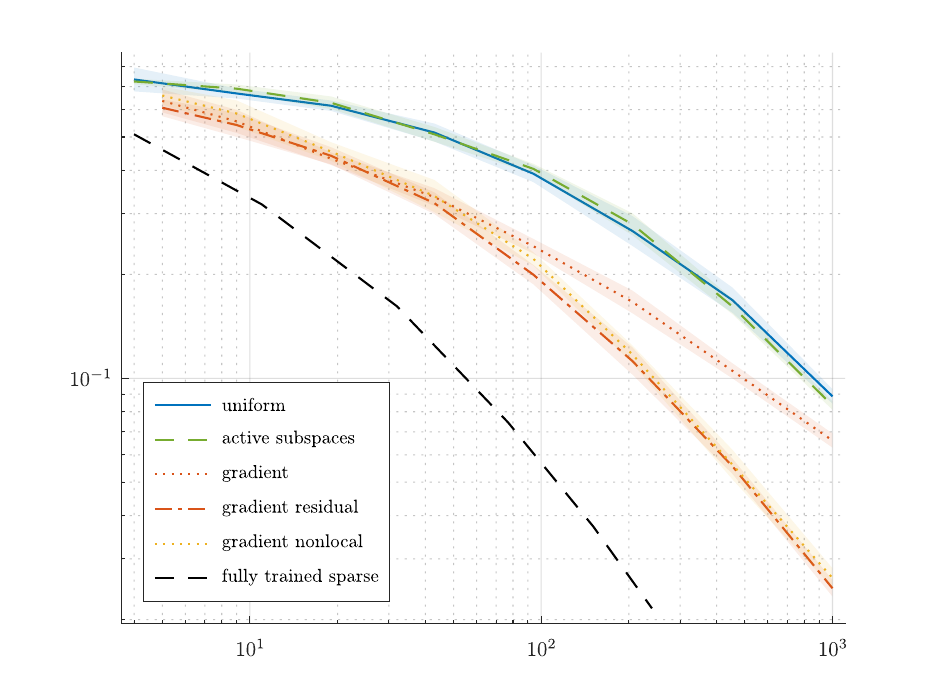}
\includegraphics[width=0.45\textwidth]{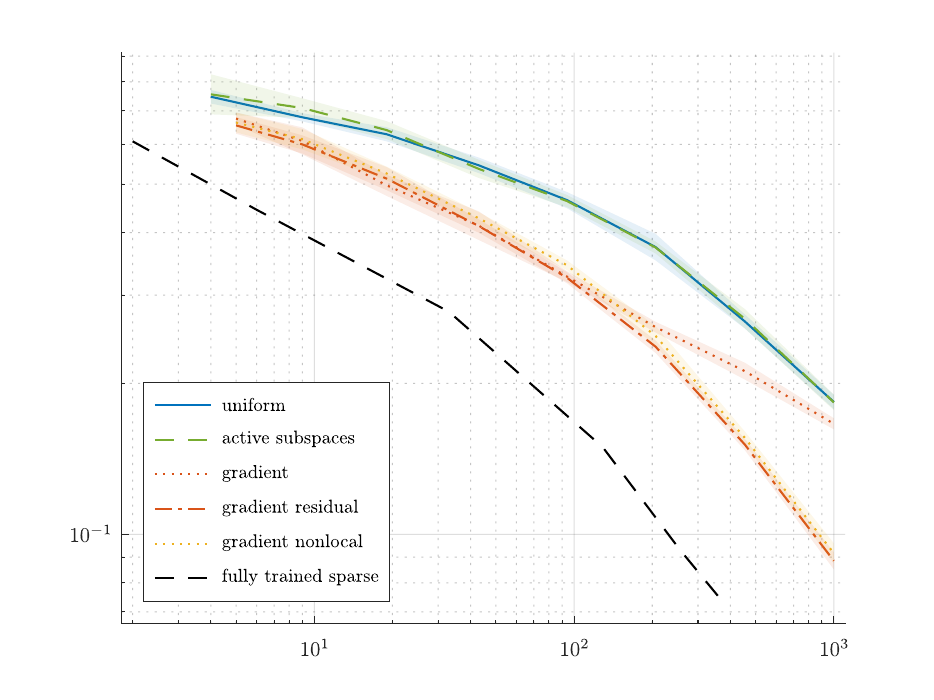}  
\caption{Convergence plots for locally anisotropic Gaussian ``checkmark'' function in three (left) and four (right) dimensions.}
    \label{fig:convergence_checkmark_hd}
\end{figure}
We will give two additional examples that illustrate the major weaknesses of the simple gradient based approach~\eqref{eq:gradient_density_discrete}.


\subsubsection{Nonsmooth functions}
The next two examples test functions which do not fulfill the requirements of parts of our theory, and highlight the limitations of some or all developed methods. First, we test the nonsmooth function
\[
    f(x) = \max\{\,0,\, x_1,\, x_2\,\},
\]
which has only three gradients \(\nabla f(x) \in \{\,(1,0)^T,(0,1)^T,(0,0)^T\,\}\), with exception of the points 
where the function is not differentiable and the behavior of the methods is implementation specific.
Thus the simple density~\eqref{eq:gradient_density_discrete} will not be appropriate, but the behavior of other densities is less obvious. In Figure~\ref{fig:convergence_2d_counter}, we observe that indeed this density leads to a stagnation of the error with increasing number of samples. However, both the nonlocal density~\eqref{eq:gradient_nonlocal_density_discrete} and the residual strategy from section~\ref{sec:nonuniform_residual} are able to mitigate this and perform on par with uniform sampling. Only the integral formula~\eqref{eq:approximate_density_MC} and the residual based method outperform the uniform sampling by a small margin. Corresponding densities are given in Figure~\ref{fig:density_2d_counter}.

\begin{figure}
    \includegraphics[width=0.45\textwidth]{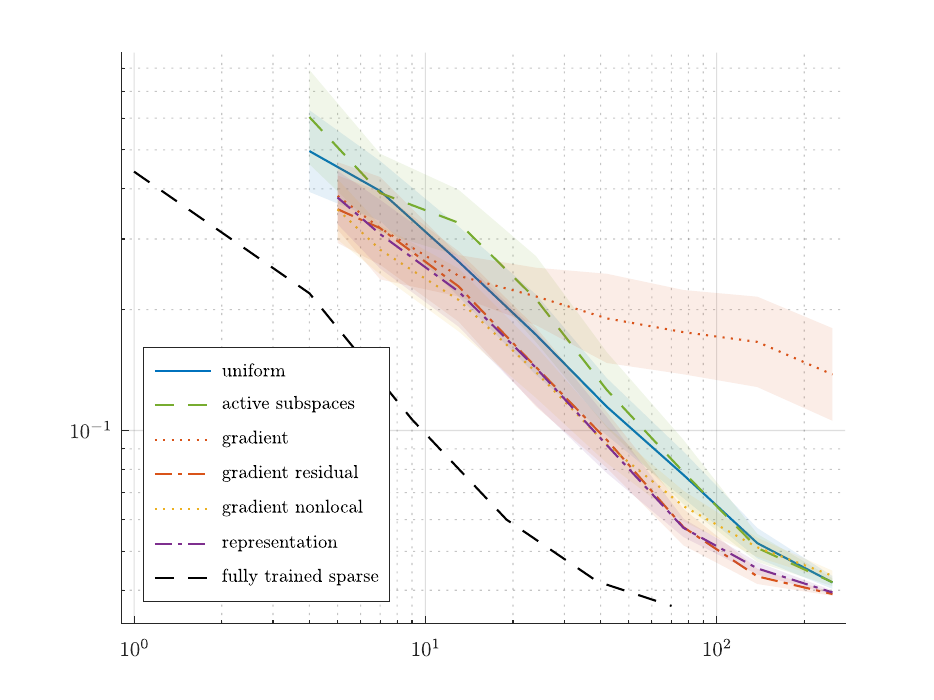}
    \includegraphics[width=0.45\textwidth]{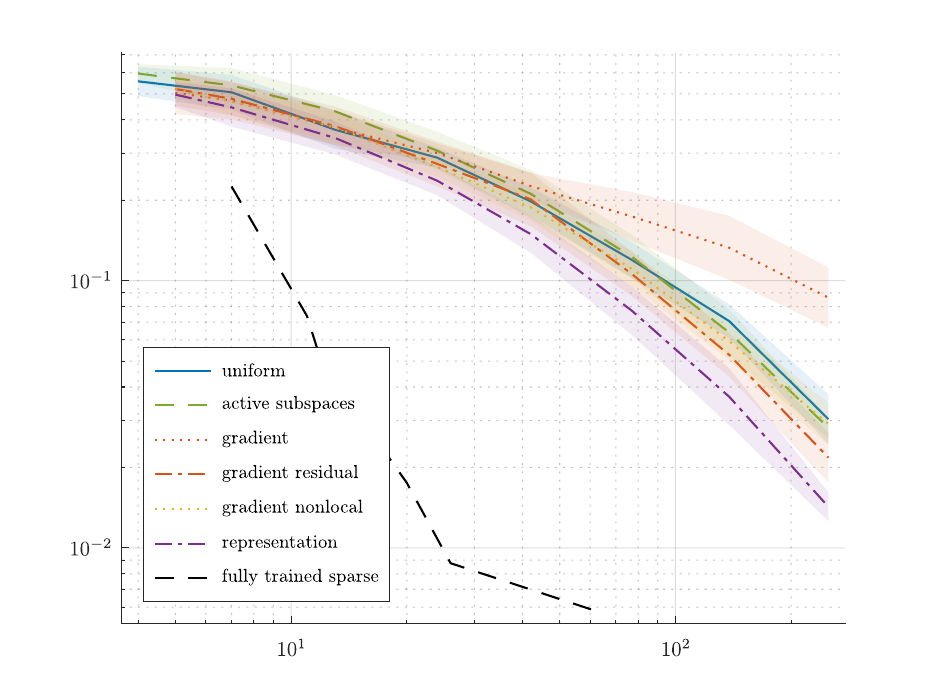}
    \caption{Convergence plots for two dimensions; left: non-smooth corner function, right: sum of two planar functions.}
    \label{fig:convergence_2d_counter}
\end{figure}

\begin{figure}
    \includegraphics[width=0.32\textwidth]{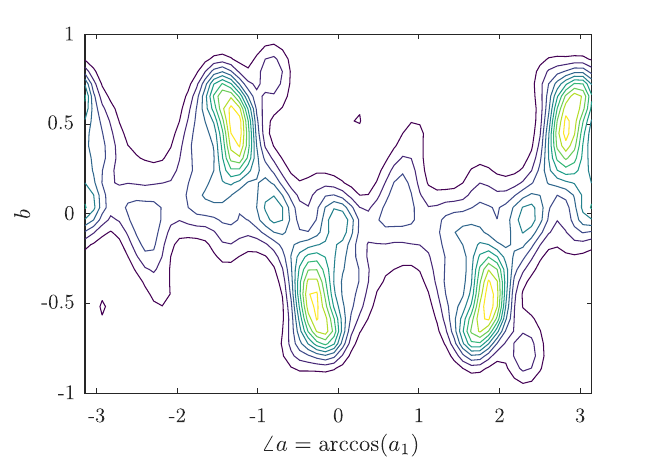}
    \includegraphics[width=0.32\textwidth]{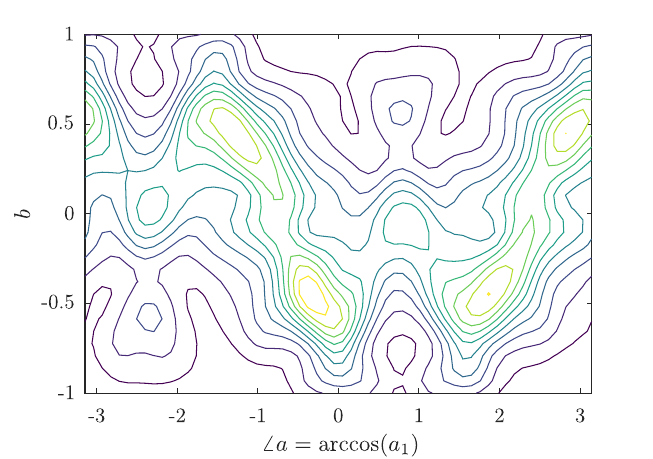}
    \includegraphics[width=0.32\textwidth]{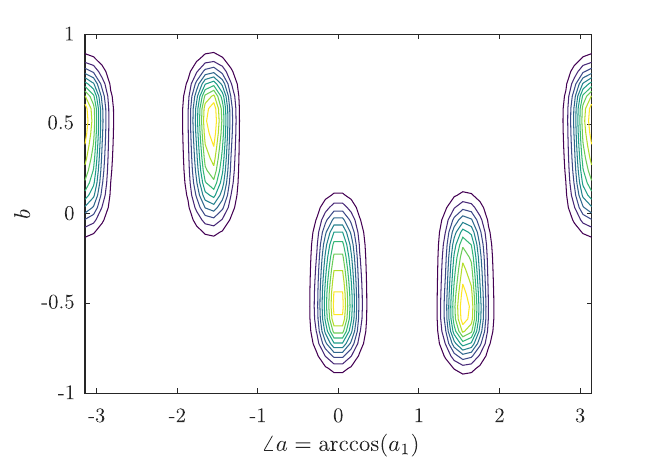}
    \includegraphics[width=0.32\textwidth]{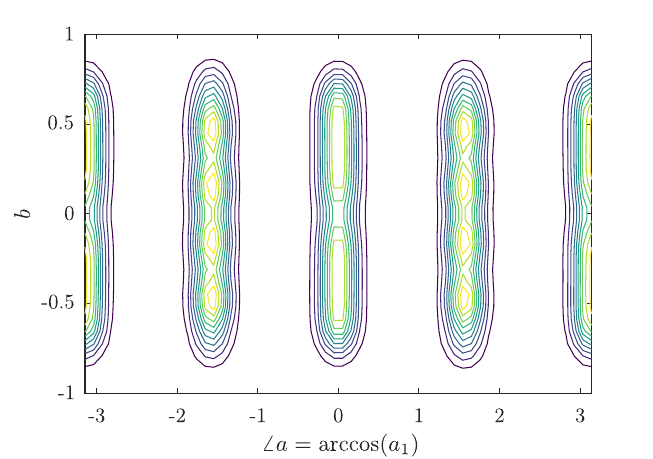}
    \includegraphics[width=0.32\textwidth]{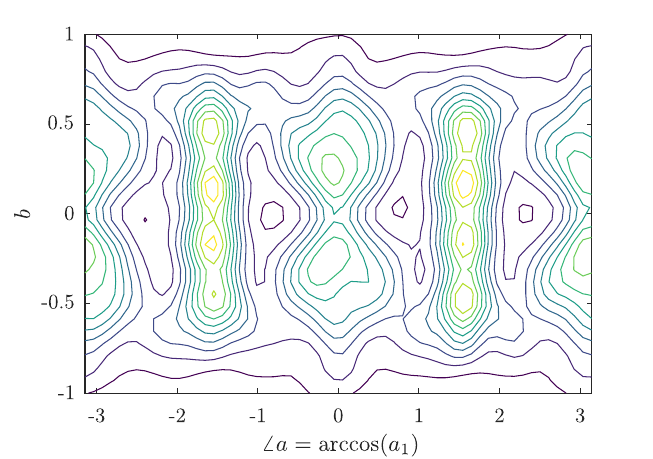}
    \includegraphics[width=0.32\textwidth]{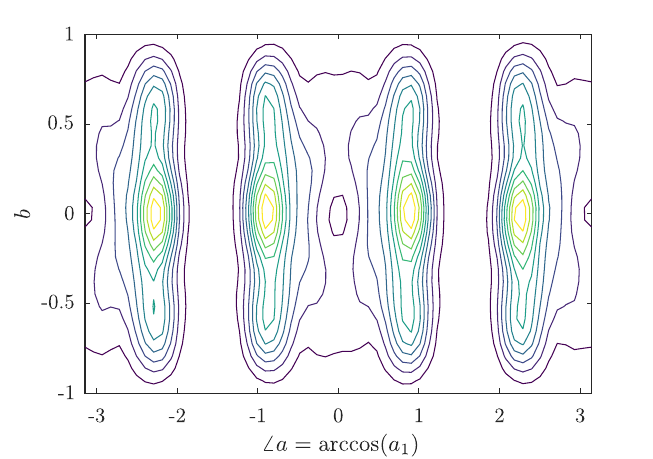}
    \caption{Density of different distributions on parameter space (KDE estimate); left to right: density from sparse feature learning, integral density, and gradient based sampling; top: non-smooth corner function, bottom: sum of two planar functions.}
    \label{fig:density_2d_counter}
\end{figure}

\subsubsection{Separable functions}
Another class of functions that have simple sparse representations are sums of planar functions (separable functions). Here we consider the example
\[
    f(x) = \cos(10 x_1) + h(7 x_2),
\]
with the previously introduced Huber function \(h\),
which has the gradient \(\nabla f(x) = -10\sin(10 x_1) (1,0)^T + 7 h'(7 x_2)(0,1)^T\).
Although the gradient is the sum of only two global directions, the largest gradient values are given for \(x_1 \approx \pm\pi/20\) and \(\abs{x_2} > 1/14\), where \(\nabla f(x) \approx (\mp 10,7)^T\). This is indeed the dominant direction produced by the simple gradient based sampling strategy; see Figure~\ref{fig:density_2d_counter} on the right. However, the two directions that allow for an efficient representation of the underlying functions \((1,0)^T\) and \((0,1)^T\) are only identified correctly by sparse feature learning. Only the approximate integral formula peaks around these weights, but is affected by substantial blurring and artifacts, since the underlying assumptions for validity of the formula are not fulfilled. This also explains the behavior of the different strategy in terms of approximation quality over the number of degrees of freedom; see Figure~\ref{fig:convergence_2d_counter}.
The sparse feature learning converges quickly, since it can exploit that \(f\) is the sum of two planar (effectively one-dimensional) functions.
The random feature models are not able to exploit this, with the possible exception of the integral formula, which has some small advantage compared to other random methods, but is substantially less accurate than fully trained weights.

\subsection{Established high-dimensional testcases}

Finally, we test the methods on some established high dimensional benchmark tests, taken from~\cite{simulationlib}.
Here we generally transform the input data by scaling and shifting so that it tightly fits inside the unit sphere, and use \(K=5000\) i.i.d.\ uniform random training and validation samples, respectively. This is only necessary for the easy comparison to the uniform random weight sampling from section~\ref{sec:uniform_random_features}, where we set \(R=1\). We use \(s=1\) and \(\delta = 1/40\).

\subsubsection{Corner peak}
We consider the ``corner peak'' benchmark~\cite{Genz:1984}:
\[
  f(x) = 
  10\left(1 + a\sum_{i=1}^d x_i\right)^{-d-1},
  \quad
  x \in [0,1]^d, a > 0.
\]
The results for \(a=2\) in three and four dimensions are given in Figure~\ref{fig:convergence_corner_peak_hd}.
Similar to the two dimensional planar wave, we observe a big difference between methods that are able to identify the ``active subspace'', with an additional advantage for methods that can place more hyperplane close to the peak in the corner.
\begin{figure}
    \includegraphics[width=0.45\textwidth]{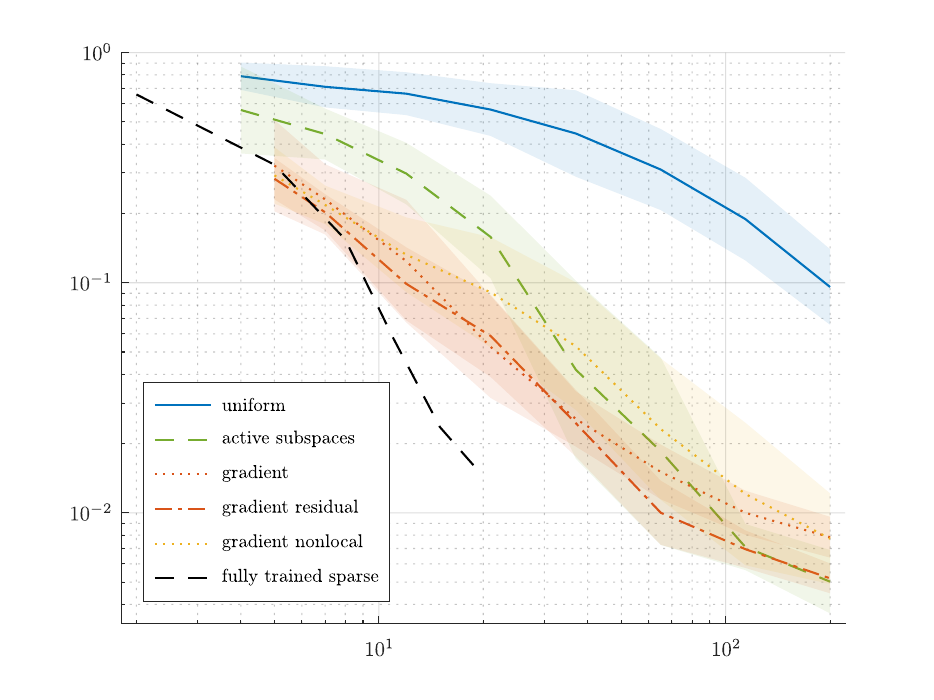}
    \includegraphics[width=0.45\textwidth]{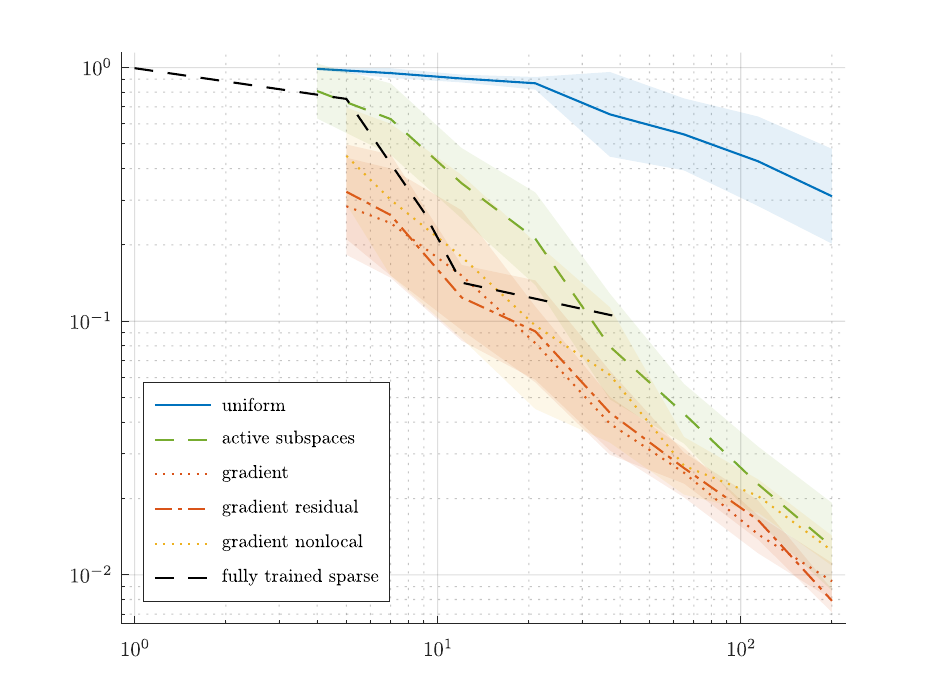}
    \caption{Convergence plots for the corner peak in three (left) and four (right) dimensions.}
    \label{fig:convergence_corner_peak_hd}
\end{figure}

\subsubsection{Robot arm and borehole functions}

Finally we consider two moderately high dimensional established benchmarks motivated by applications to see if the differences between methods observed in the manufactured cases transfer to those examples.
First, we consider the ``Robot arm'' testcase (cf.~\cite{An:2001}). This describes the length of a robot arm with \(d/2\) segments of length \(x_{2i} = L_i\) at relative angles \(x_{2i+1} = \theta_i\) for \(i=1,\ldots,d/2\).
\[
  f(x) = \sqrt{u^2 + v^2};
\quad
  u = \sum_{j=1}^{d/2} L_i \cos(\Theta_i),
  v = \sum_{j=1}^{d/2} L_i \sin(\Theta_i),
\quad
  \Theta_i = \sum_{j=1}^{i} \theta_j,
\]
\begin{figure}
    \includegraphics[width=0.45\textwidth]{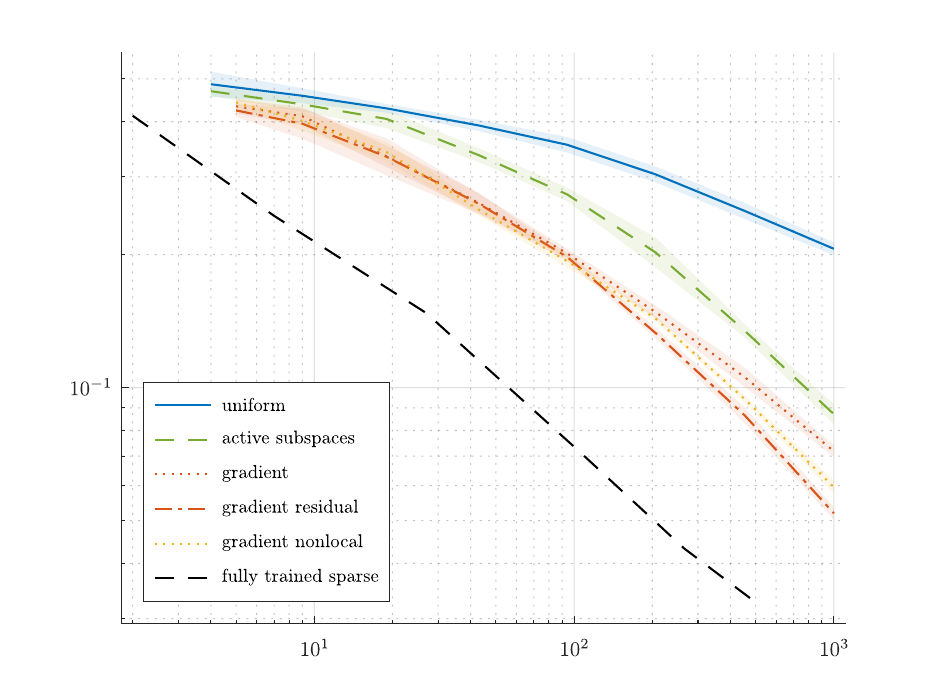}
    \includegraphics[width=0.45\textwidth]{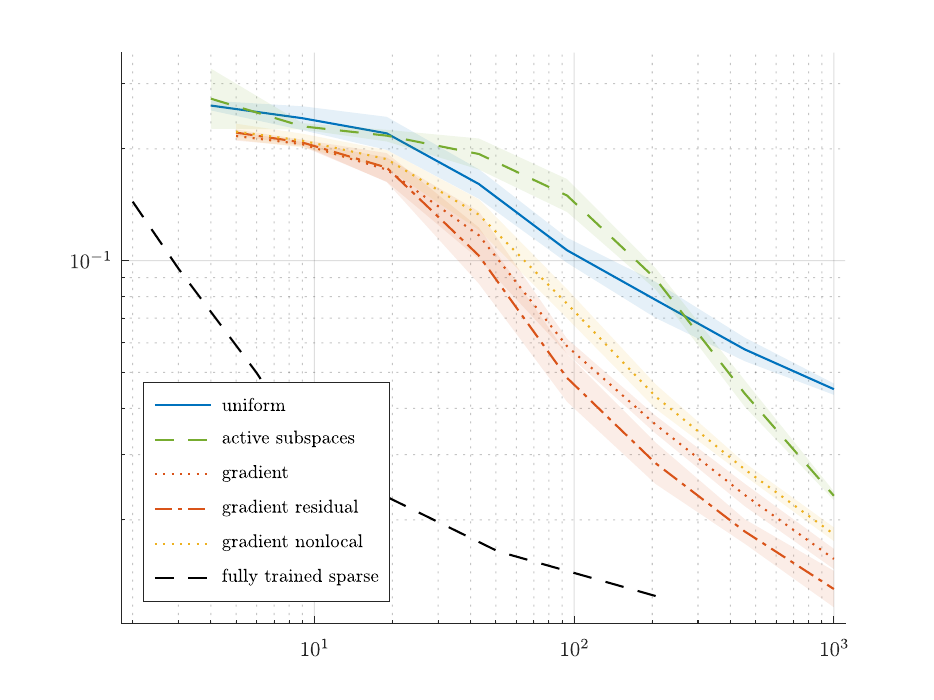}
    \caption{Convergence plots: for the Robot arm in six dimensions (left); for the eight dimensional borehole function (right).}
    \label{fig:convergence_hd}
\end{figure}
Second, we consider the ``borehole'' function proposed in~\cite{Morris:1993}, which is eight dimensional. The results are given in Figure~\ref{fig:convergence_hd}.
In general, we observe a similar behavior as suggested by the previous test cases. In the robot arm case, there is a significant improvement from the identification of the globally dominant direction (as in active subspaces), but further improvements are obtained by taking location of the hyperplanes into account. In the borehole example, there is a more substantial difference between globally and locally anisotropic weight sampling. For low number of weight samples, the active subspaces based approach is not able to outperform uniform sampling, with improvements visible only for more than \(\sim 300\) neurons. However, the localized hyperplane sampling apparently can predict better the location of hyperplanes better with respect to where the variability is located in the input space, and perform consistently better than anisotropic or isotropic uniform weight sampling. 

\section*{Acknowledgement}
This material is based upon work supported by the U.S. Department of Energy, Office of Science, Office of Advanced Scientific Computing Research, Applied Mathematics program under the contract ERKJ387 at the Oak Ridge National Laboratory, which is operated by UT-Battelle, LLC, for the U.S. Department of Energy under Contract DE-AC05-00OR22725.

\bibliography{lit}
\bibliographystyle{abbrv}

\newpage
\appendix

\end{document}